\documentclass[letterpaper, 10 pt, journal, twoside]{IEEEtran}

\ifCLASSINFOpdf

\else
 
\fi

\usepackage{amsfonts}
\usepackage{amsmath}
\usepackage{mathrsfs}
\usepackage{wrapfig}
\usepackage{multirow}
\usepackage{mathtools}
\usepackage{array}

\usepackage{booktabs}
\usepackage[caption=false,font=footnotesize]{subfig}
\usepackage{setspace}              
\usepackage[ruled,linesnumbered]{algorithm2e}
\usepackage{siunitx}
\usepackage{placeins}   
\usepackage{graphicx}
\usepackage{microtype} 
\usepackage[bottom,flushmargin]{footmisc} 

\usepackage{pifont} 

\usepackage{multibib}
\newcites{app}{References (Appendix)}

\usepackage[dvipsnames,table]{xcolor}

\newcommand{\mycheck}{{\color{ForestGreen}\ding{51}}}
\newcommand{\myxmark}{{\color{red}\ding{55}}}

\usepackage{makecell}  

\newcommand{\ccell}[2]{\cellcolor{#1}\rule{0pt}{2.4ex}#2\rule[-1.2ex]{0pt}{0pt}}

\usepackage{dirtytalk} 

\usepackage{balance} 
\usepackage{flushend}
\usepackage{hyperref} %
\hypersetup{hidelinks}
\usepackage{icomma} 

\newcolumntype{G}{>{\columncolor{green!15}}c} 
\newcolumntype{R}{>{\columncolor{red!15}}c} 

\newtheorem{theorem}{Theorem}

\newtheorem{lemma}{Lemma}
\newtheorem{corollary}{Corollary}

\newtheorem{assumption}{Assumption}
\newcommand{\transpose}{^\mathsf{T}}

\newcommand{\dynamics}{f}
\newcommand{\approxDynamics}{\tilde f}

\newcommand{\SAFE}{\mathcal Q_{safe}}
\newcommand{\dt}{\Delta t}
\newcommand{\noiseDist}{P_{noise}}

\newcommand{\CPmodel}{\texttt{model}}
\newcommand{\CPmodelinput}{x}
\newcommand{\CPmodelpred}{h}
\newcommand{\CPoutput}{y}
\newcommand{\CPpredRegion}{\mathcal C}
\newcommand{\CPpredRegionPrime}{\mathcal C^{\prime}}
\newcommand{\CPpredRegionExp}{\mathcal C^{\exp}}
\newcommand{\CPpredRegionQ}{\mathcal C^{q}}

\newcommand{\CPmodelInputSpace}{\mathcal X}
\newcommand{\CPmodelPredSpace}{\mathcal H}
\newcommand{\CPoutputSpace}{\mathcal Y}
\newcommand{\CPdist}{P_{data}}

\newcommand{\Lagrangian}{\mathcal L}

\newcommand{\SEtwo}{$SE(2)$}
\newcommand{\SEthree}{$SE(3)$}
\newcommand{\lGroup}{\mathcal G}
\newcommand{\lAlgebra}{\mathfrak g}

\newcommand{\lGroupEle}{g}
\newcommand{\lCoad}{\textrm{ad}^*_\xi}

\newcommand{\lTwistMap}{\mathcal J_K}
\newcommand{\lInertiaMatrix}{\mathcal M}

\newcommand{\constMatrix}{A}
\newcommand{\lConstMatrix}{\mathcal A}
\newcommand{\inputMatrix}{B}
\newcommand{\lInputMatrix}{\mathcal B}

\newcommand{\FKmap}{K}

\newcommand{\methodName}{\textbf{CLAPS}}
\newcommand{\longMethodName}{\textbf{C}onformal \textbf{L}ie-group \textbf{A}ction \textbf{P}rediction \textbf{S}ets }

\makeatletter
\def\IEEEproofindentspace{\parindent}
\makeatother

\makeatletter
\@ifundefined{proof}{%
  \@ifundefined{IEEEproof}{%
    \newenvironment{proof}[1][Proof]{%
      \par\noindent\textit{#1: }\ignorespaces
    }{%
      \hspace*{\fill}$\square$\par
    }%
  }{%
    \newenvironment{proof}{\IEEEproof}{\endIEEEproof}%
  }%
}{}
\makeatother

\begin{document}

\title{Lies We Can Trust: Quantifying Action Uncertainty with Inaccurate Stochastic Dynamics through Conformalized Nonholonomic Lie groups}
\author{Lu\'is Marques$^{1}$, Maani Ghaffari$^{1}$, and Dmitry Berenson$^{1}$
\thanks{Accepted for publication in IEEE Robotics and Automation Letters (RA-L). Published in IEEE Robotics and Automation Letters, vol. 11, no. 4, pp. 4801--4808, Apr. 2026. DOI: 10.1109/LRA.2026.3656773. Manuscript received September 22, 2025; revised November 19, 2025; accepted January 7, 2026.}
\thanks{\copyright~2026 IEEE. Personal use of this material is permitted.  Permission from IEEE must be obtained for all other uses, in any current or future media, including reprinting/republishing this material for advertising or promotional purposes, creating new collective works, for resale or redistribution to servers or lists, or reuse of any copyrighted component of this work in other works.}
\thanks{This paper was recommended for publication by Editor Lucia Pallottino upon evaluation of the Associate Editor and Reviewers' comments.
This work was supported in part by the Office of Naval Research
Grant N00014-24-1-2036 and NSF grants IIS-2113401 and IIS-2220876. M. Ghaffari was supported by AFOSR MURI FA9550-23-1-0400 and AFOSR YIP FA9550-25-1-0224.} 
\thanks{$^{1}$Lu\'is Marques, Maani Ghaffari and Dmitry Berenson are with the Robotics Department, University of Michigan, Ann Arbor, MI 48109 USA
        {\tt\footnotesize lmarques@umich.edu; maanigj@umich.edu; dmitryb@umich.edu}}
}

\markboth{IEEE Robotics and Automation Letters. Accepted January 2026}
{Marques \MakeLowercase{\textit{et al.}}: Quantifying Action Dynamics Uncertainty through Conformalized Nonholonomic Lie groups}

\maketitle

\begin{abstract}
We propose \longMethodName (\methodName), a symmetry-aware conformal prediction-based algorithm that constructs, for a given action, a set guaranteed to contain the resulting system configuration at a user-defined probability. 
Our assurance holds under both aleatoric and epistemic uncertainty, non-asymptotically, and does not require strong assumptions about the true system dynamics, the uncertainty sources, or the quality of the approximate dynamics model.
Typically, uncertainty quantification
is tackled by making strong assumptions about the error distribution or magnitude,
or by relying on uncalibrated uncertainty estimates --- i.e., with no link to frequentist probabilities --- which are insufficient for safe control.
Recently, conformal prediction has emerged as a statistical framework capable of providing distribution-free probabilistic guarantees on test-time prediction accuracy.
While current conformal methods treat robot configurations as Euclidean points, many systems have non-Euclidean configurations, e.g., some mobile robots have \SEtwo.
In this work, we rigorously analyze configuration errors using Lie groups, extending previous Euclidean space theoretical guarantees to \SEtwo. Our experiments on a simulated JetBot, and on a real MBot, suggest that by considering the configuration space's structure, our symmetry-informed nonconformity score leads to more volume-efficient prediction regions which represent the underlying uncertainty better than existing approaches. Website: \url{https://um-arm-lab.github.io/claps}
\end{abstract}


\IEEEpeerreviewmaketitle

\section{Introduction}

\IEEEPARstart{R}{obotic} systems often operate in stochastic environments while relying on imperfect dynamics models for control.
Inaccuracies in predicting future robot configurations can arise due to \textit{epistemic uncertainty} --- resulting from limited information, changes in robot-environment interactions (e.g., wear, new terrain), model simplifications (e.g., no-slip assumption), input delays, etc. --- and as a consequence of \textit{aleatoric uncertainty} --- irreducible stochasticity  (e.g., external disturbances).
Learning-based methods have been increasingly used to tackle uncertainty-rich control problems \cite{shi2019neural,salzmann2023real}, but can lack provable safety assurances due to using \textit{uncalibrated} uncertainty estimates, i.e., that cannot be interpreted as likelihoods. For example, the $90\%$ confidence region of a Gaussian Process can contain the true labels at a lower likelihood.
Conversely, traditional safety-critical tools might make strong distributional error assumptions---neglect stochasticity, assume disturbances are Gaussian or have fixed known bounds \cite{jeantube,agrawal2022safe}---which may not hold in practice.
Providing rigorous test-time distribution-free \textit{calibrated uncertainty predictions}, containing the true unobserved labels at the specified likelihood, when given inaccurate dynamics models and subject to uncharacterized external disturbances, remains an active research problem.
\begin{figure}[t!]
  \centering
  \includegraphics[width=\linewidth]{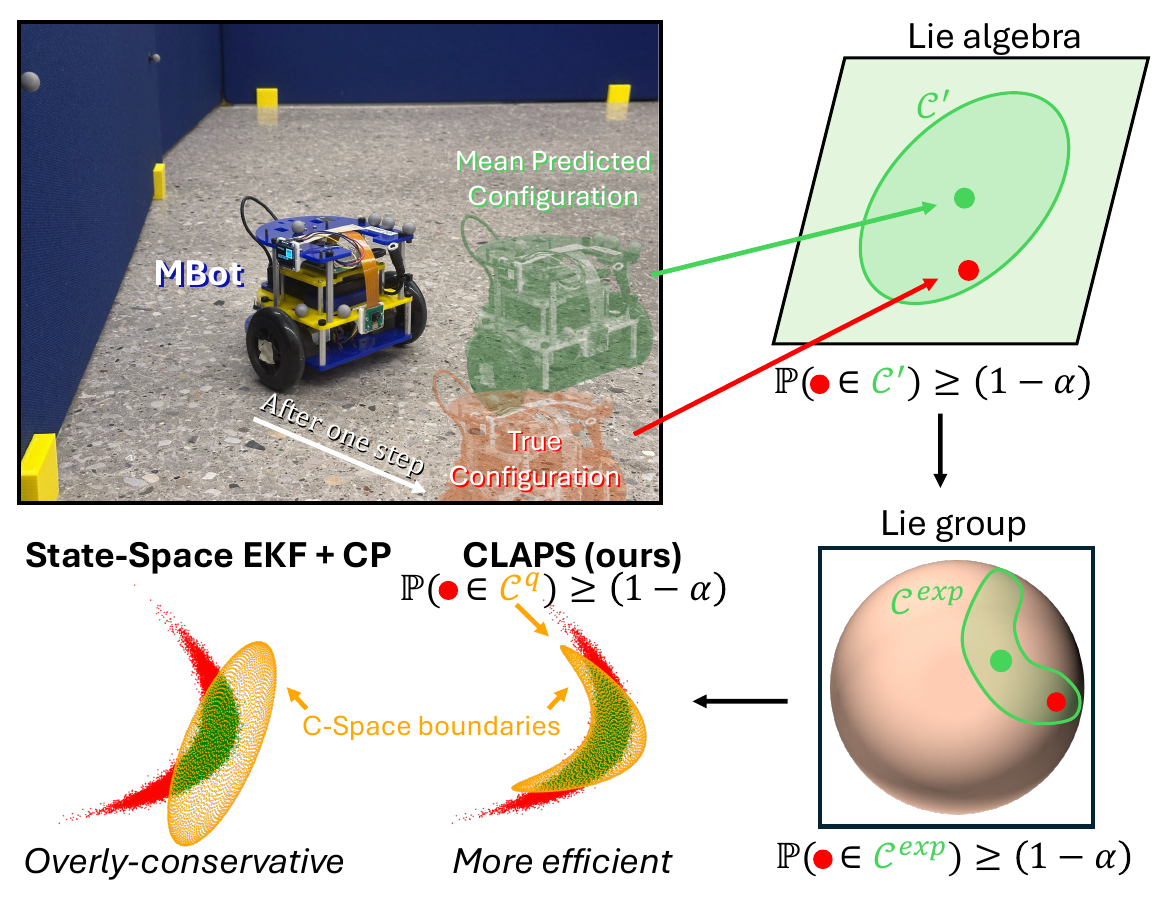}
  \caption{Our proposed algorithm (\methodName) constructs prediction regions $\CPpredRegionQ$ (in C-space) that are \textit{marginally guaranteed} to contain the next \textit{unknown system configuration} at a user-set probability $(1-\alpha)$. By considering the robot's symmetry, we can construct more \textit{efficient} prediction regions.}
   \label{fig:title_figure}
   \vspace{-8mm}
\end{figure}
\newline\indent
\looseness-1
Conformal Prediction (CP) has recently emerged as a promising framework for tackling this problem non-asymptotically, enabling the construction of prediction regions containing the true system state at a given user-set probability---with the assured probability being defined over the distribution of test-time conditions and not for a specific test scenario.
However, despite many robot configuration spaces being non-Euclidean (e.g., \SEtwo~or \SEthree), current CP methods treat states as Euclidean vectors and consider point-prediction models \cite{lindemann2023safe,dixit2023adaptive}, resulting in convex prediction regions that can be overly-conservative --- \textit{inefficient} in terms of the volume covered to achieve the desired probability.
This lack of prediction efficiency limits the downstream applicability of CP regions for safe planning and control, e.g., if an action is deemed probabilistically safe when its resulting prediction region is collision-free, then region inefficiency reduces the set of safe actions, potentially also impacting task speed and other metrics.
There is an opportunity to derive \textit{tighter} prediction regions by accounting for the configuration space's inherent symmetry in a theoretically grounded way.
In robot localization, Lie groups have been used to provide geometrically-aware state estimators \cite{barrau2016invariant,barfoot2024state}, improving convergence guarantees and speed. Yet, to our knowledge, CP has not been used with Lie groups to produce prediction regions in non-Euclidean configuration spaces.
\newline\indent
We propose \longMethodName (\methodName), a CP-based algorithm that uses a dataset of state transitions to calibrate the uncertainty estimates provided by  approximate dynamics models. 
 \methodName~can be applied as a \textit{post-hoc calibration layer} on top of existing Lie-algebraic Gaussian uncertainty estimators, turning approximate covariances into \textit{provably calibrated} ones.
By deriving a Lie-based symmetry-respecting score metric, our calibration process produces regions that are smaller than Euclidean baselines, while still containing the true configuration at the user-defined probability.
Our main contributions are:
\begin{enumerate}
    \item \looseness-1 We 
    explain the results of \cite{bloch2009quasivelocities,bloch2015nonholonomic}, on converting nonholonomic dynamics from State Space (SS) to Lie group form, in a concise, self-contained, and clear manner.
    
    \item \looseness-1 We introduce an algorithm that, given an approximate dynamics model estimating prediction uncertainty as Gaussian, constructs state- and action-dependent \textit{calibrated prediction sets} in \SEtwo~ that provably (marginally) contain the resulting configuration --- despite aleatoric \textit{and} epistemic uncertainty.
    \item We perform simulation and hardware experiments to support our theoretical claims and demonstrate an increase in prediction region \textit{volume-efficiency} and \textit{representation quality} relative to state-of-the-art baselines. 
    \item We open-source our implementation (real and sim)\footnote{Available at: \url{https://github.com/UM-ARM-Lab/claps_code}\label{code_url}}.
\end{enumerate}

\section{RELATED WORK}\vspace{-1.5mm}

\looseness-1
Safety-critical control has a rich history, spanning a broad spectrum of methods including reachability analysis \cite{bansal2017hamilton}, tube-based MPC \cite{oestreich2023tube}, control barrier functions \cite{ames2019control}, and, more recently, learning-based approaches \cite{brunke2022safe}. 
Data-driven tools have been proposed to loosen some of the classical assumptions, enabling the estimation of epistemic uncertainty through ensembles \cite{wang2024providing} or reconstruction losses \cite{nagami2024state}, and of aleatoric uncertainty by the posterior variance of Gaussian Processes \cite{khan2021safety}.
However, this often comes at the cost of providing looser guarantees, due to a reliance on \textit{uncalibrated uncertainty estimates}.
Despite impressive results, it is still challenging to rigorously construct \textit{calibrated uncertainty regions} when the stochastic disturbances and model errors are possibly unbounded and uncharacterized. 
Our data-driven probabilistically-valid prediction sets require fewer assumptions than classical approaches and could be used e.g., to inform the safety certificates of CBFs \cite{zhang2025conformal}.

\looseness-1
CP has been increasingly used in robotics to construct marginally safe trajectories, with applications to social navigation \cite{lindemann2023safe,dixit2023adaptive}, expert imitation \cite{sun2023conformal}, barrier-function synthesis \cite{zhou2024safety}, safety-filter creation \cite{strawn2023conformal}, failure detection \cite{luo2024sample}, and state estimation through perception \cite{yang2023safe}.
By relying on calibration data to infer the accuracy of an available model, CP alleviates the need for strong distributional assumptions.
Existing approaches treat robot configurations as Euclidean vectors, defining error metrics through vector differences between point predictions and ground-truth configurations.
Ignoring the inherent symmetry of robotic systems can result in overly-conservative prediction regions, impacting downstream task performance.
In \cite{marques2024quantifying}, an uncertainty-aware calibration procedure was introduced but states were still viewed as Euclidean, limiting the prediction regions to be (convex) hyperellipsoids.
Instead, we propose a symmetry-aware error metric and uncertainty calibration procedure leading to (possibly non-convex) regions that are potentially more \textit{volume-efficient}, while still containing the true configurations at the desired probability.
Our work can be seen as \textit{complementary} to existing CP pipelines since the proposed alternative score and region construction procedure could potentially be used with other CP algorithms.

\looseness-1
Lie groups have been used to represent and propagate robotic configuration uncertainty in \cite{brossard2017unscented,schuck2024reinforcement, long2013banana}.
\looseness-1 
While the Invariant EKF (InEKF) has been empirically shown to better represent some systems' uncertainty than SS alternatives \cite{barrau2016invariant,barrau2018invariant}, its uncertainty estimates are still \textit{uncalibrated}, not being sufficient for provably safe control. We provide probabilistic guarantees despite model mismatch.
Often Lie group motion is treated as quasi-static, with actions being delta-poses.
Instead, we consider the impact of disturbances and analyze dynamical systems, which require a different formulation respecting inertia. 
While the control community has used Lie group dynamics for reference trajectory tracking \cite{jang2023convex,tayefi2019logarithmic} and wheel slippage estimation \cite{yu2023fully}, current methods enforce motion constraints in heuristic and system-specific ways (e.g., by zeroing out specific terms), decoupling the impact of angular uncertainty in positional uncertainty. Instead, we use the theoretical contributions of \cite{bloch2009quasivelocities,bloch2015nonholonomic} --- presented concisely in $\S$\ref{sec:convert_to_EPS} --- enabling the conversion of some nonholonomic dynamical systems to Lie group form and the propagation of uncertainty along the constraint manifold.

\vspace{-1mm}
\section{PROBLEM 
STATEMENT}\label{sec:prob_statement}\vspace{-2mm}

Let $q \in \mathcal Q$ be an $n$-dimensional vector denoting the C-space configuration (in \textit{generalized coordinates}) of a robot with $d$ degrees of freedom $(d \le n)$, $\dot q \in T_q\mathcal Q$ be the generalized velocity, and $s:=(q,\dot q) \in T\mathcal Q$ its state.
We consider both holonomic systems and robots subject to \textit{nonholonomic constraints}, i.e., non-integrable constraints on the allowable velocities.
These motion constraints are often expressed in Pfaffian form $A(q)\dot q=0$, where $A(q) \in \mathbb R^{k\times n}$ is a configuration-dependent full-rank constraint matrix of $k$ constraints (one per row, where $k<d$).
We observe the true dynamical system state at 
discrete points $s_k:=s(k\dt)$, where $k \in \mathbb N$ and $\dt$ is the sampling time between measurements.
We consider systems with time-invariant stochastic dynamics whose $\mathcal Q$ can be represented by the matrix Lie groups \SEtwo~and \SEthree.
This class is broad, encompassing unicycles, differential drive vehicles, car-like systems, quadrotors, surface vessels, underwater vehicles, satellites,  quadrupeds modeled by their center of mass, and so on.
As these matrix groups are non-compact and non-commutative, their theoretical analysis poses additional challenges compared to groups like $SO(2)$. We write the \textit{true unknown dynamics} of the observable discrete process as
\vspace{-2mm}\begin{equation}
    s_{k+1}=\dynamics(s_k, u_k, w_k), \quad w_k \sim \noiseDist \label{eq:true_dyn_discrete} ,\vspace{-2mm}
\end{equation} 
where $\dynamics$ is an \textit{unknown} deterministic function, $w_k$ is a stochastic term drawn iid from an \textit{unknown} distribution $\noiseDist$, and $u_k\in \mathbb R^m$ is the control input.
Inaccuracies in modeling $\dynamics$ may arise e.g., from domain shifts between fitting and deployment, and result in \textit{epistemic uncertainty}.
On the other hand, $w_k$ makes the dynamics stochastic, introducing \textit{aleatoric uncertainty}, and may represent external disturbances such as wind gusts, wheel slippage, or terrain bumps.
Without restricting $w_k$ or $\dynamics$, we only consider approximate discrete models $\approxDynamics$ where the prediction uncertainty is modeled as Gaussian.
We can then write $
    (\tilde s_1, \tilde \Sigma_1) = \approxDynamics (s_0, u_0),$ where $\tilde s_1:=\mathbb E(\approxDynamics (s_0, u_0))$ and $\tilde \Sigma_1$ is a covariance matrix. 
To make uncertainty quantification tractable, without imposing distributional error assumptions, we assume access to an uncorrupted dataset of state transitions. 
\begin{assumption} \label{ass:ex} We are given a dataset of transitions \mbox{$D_{cal}=\{(s_k, u_k, s_{k+1})\}_{1:N}$}, collected from the same transition distribution $\CPdist$ observed at execution time.
\end{assumption}

Formally, $D_{cal}$ is exchangeable with the test-time transitions $(s_k,u_k,s_{k+1})$\footnote{A random vector $D_{cal}\cup(s_k,u_k,s_{k+1}):=(s_k,u_k,s_{k+1})_{1:N+1}$ is exchangeable if its elements are equally likely to appear in any ordering.}, which is a weaker requirement than iid (iid $\Rightarrow$ exchangeable).
Hence, the $D_{cal}$ transitions cannot, for example, be collected along a single robot trajectory, as then they would be time-correlated and no longer exchangeable.
While we implicitly assume access to test-time conditions (to sample from $\CPdist$) and that $D_{cal}$ can be safely collected, our theoretical guarantees are non-asymptotic in nature.

Our goal is to provide, for a given admissible action $u_{des}$, a C-space prediction region $\CPpredRegionQ \subseteq \mathcal Q$ that provably contains the resulting true \textit{unknown} system configuration $q_1$ with, at least, a user-defined probability $(1-\alpha)$, i.e. \vspace{-2mm}\begin{equation}
    \mathbb P(q_1 \in \CPpredRegionQ) \ge (1-\alpha),\label{eq:safe_condition} \vspace{-2mm}
\end{equation}
\looseness-1 where $\alpha \in (0,1)$ is the user-set acceptable failure-probability.
Following CP literature, the likelihood is taken on average over the test-time scenarios, not for a specific $q_1$, and we assume the initial system state to be known (i.e., $\tilde s_0=s_0$).
While purely achieving \eqref{eq:safe_condition} is trivial, e.g., by predicting the entire space $\CPpredRegionQ = \mathcal Q$, we additionally want $\CPpredRegionQ$ to be \textit{as tight/volume-efficient as possible} to make it practical for downstream robotic tasks such as safe control.
This is a challenging problem. We consider both \textit{aleatoric and epistemic uncertainty} and do not make strong assumptions about the fidelity of $\approxDynamics$ or the nature of the stochastic disturbances.
While we make no claims about how efficient our prediction regions are, we show they can be tighter than existing methods.

\vspace{-2mm}
\section{BACKGROUND}\vspace{-2mm}
Before introducing our method, let us briefly cover the background material needed to prove our main contributions. 

\vspace{-1mm}
\subsection{Lagrangian Mechanics for Nonholonomic Systems}
\vspace{-1mm}

A robot's Lagrangian $\Lagrangian:T\mathcal Q\to \mathbb R$ is often given by $\Lagrangian(q,\dot q):= T(q,\dot q)- V(q)$, where $T$ and $V$ are the kinetic and potential energies respectively
\cite{bloch2015nonholonomic,marsden2013introduction}.
Then, the dynamics of a nonholonomic system are given by the \textit{forced Lagrange-d'Alembert equations}
\cite{bloch2015nonholonomic}:
$
    \frac{d}{dt}\frac{\partial \Lagrangian}{\partial \dot q} = \frac{\partial \Lagrangian}{\partial q} +  \inputMatrix(q) u + A(q)\transpose \lambda, \quad \constMatrix(q)\dot q=0,
$
where the force map $B(q) \in \mathbb R^{n\times m}$ converts control inputs $u$ to generalized forces and $\lambda \in \mathbb R^k$ holds Lagrange multipliers required to enforce the constraints.

\subsection{Dynamics of Nonholonomic Matrix Lie group Systems}\vspace{-1mm}

Consider that $q$ can be described by a $d$-dimensional matrix Lie group $\lGroup$. Let $T_\lGroupEle\lGroup$ denote the tangent space at group element $\lGroupEle\in \lGroup$.
The tangent space at the identity element $\mathfrak e$, $T_{\mathfrak e}\lGroup$, is called the Lie algebra $\lAlgebra$.
The vee operator $(\cdot )^\vee: \mathfrak g  \rightarrow \mathbb R^d$ denotes an isomorphism between $\lAlgebra$ and a $d$-dimensional Euclidean vector space $\mathbb R^d$.
The wedge operator denotes the inverse isomorphism $(\cdot )^\wedge :\mathbb R^d \rightarrow \lAlgebra$.
We term that $\xi \in \mathbb R^d$ is expressed in \say{exponential} coordinates.
For $SE(i)$, both $\lGroupEle \in \lGroup$ and $\xi^\wedge \in \lAlgebra$ are representable by $\mathbb R^{(i+1)\times (i+1)}$ matrices and the underlying system has $d=\dim(SE(i))=i +\frac{1}{2}i(i-1)$ DOFs.
For convenience, we denote the left group multiplication map as $L_g:h\mapsto gh, \forall g,h\in \lGroup$ \cite{lee2010introduction}.
The conventional matrix exponential maps from algebra to group $\exp_m:\lAlgebra \to \lGroup$ and the matrix logarithm from group to algebra $\log_m :\lGroup \to \lAlgebra$.
Then, the map from exponential vectors to group elements can be written as $\exp: \mathbb R^d \rightarrow \lGroup, \xi \mapsto\exp_m(\xi^\wedge)=\lGroupEle$ and the map from group elements to exponential vectors by $\log:\lGroup \rightarrow \mathbb R^d,\lGroupEle\mapsto\log_m(\lGroupEle)^\vee=\xi$. 
The \textit{ad operator} $\textrm{ad}_\xi : \lAlgebra \to \lAlgebra$ is a linear map describing how elements of $\lAlgebra$ act on each other \cite{marsden2013introduction}. It can be defined as
$
    \textrm{ad}_\xi(\eta)=[\xi^\wedge,\eta^\wedge]= \xi^\wedge\eta^\wedge -\eta^\wedge\xi^\wedge
$,
where $\xi^\wedge,\eta^\wedge\in \lAlgebra$ and $[\cdot,\cdot]$ is the matrix Lie bracket \cite{bloch2015nonholonomic}.

\looseness-1 Let $\Lagrangian:T\lGroup \rightarrow \mathbb R$ denote a left-invariant Lagrangian and $l: \lAlgebra \rightarrow \mathbb R$ its reduction to $\lAlgebra$, i.e., $l(\xi):=\Lagrangian(\mathfrak e,\xi)$.
When $l$ includes only kinetic energy, it becomes 
$l(\xi)=\frac{1}{2}\xi\transpose \lInertiaMatrix \xi$,
where $\xi \in \mathbb R^d$ is the twist in body-fixed coordinates and $\lInertiaMatrix \in \mathbb R^{d\times d}$ is the generalized inertia matrix. The twists in a controlled system evolve according to the \textit{forced Euler-Poincar\'e (EP) equations} \cite{bloch2015nonholonomic}:
$
    \frac{d}{dt}\frac{\partial l}{\partial \xi} = \lCoad\left(\frac{\partial l}{\partial \xi}\right) + \lInputMatrix u, 
$
where $u\in \lAlgebra^*$ is the control input, $\lInputMatrix$ maps $u$ to body-fixed forces or torques, $\lCoad : \lAlgebra^* \to \lAlgebra^*$ is the dual of the \textit{ad operator}, and $\lAlgebra^*$ is the Lie coalgebra\footnote{One can think of $\xi\in \lAlgebra$ as $\mathbb R^{d\times 1}$ column vectors and $u\in\lAlgebra^*$ as $\mathbb R^{1\times d}$ row vectors.
Following past literature, we write both elements as column vectors, since $\lAlgebra \cong \lAlgebra^*$, implying that $\langle u,\xi \rangle := u\transpose \xi \in \mathbb R$.}.
For matrix Lie groups, we have the useful property that $\lCoad = \textrm{ad}_\xi\transpose$ \cite{bullo2019geometric}. For the reduced Lagrangian described, the EP Eqs. become $\lInertiaMatrix \dot \xi =\textrm{ad}_\xi\transpose \lInertiaMatrix \xi + \lInputMatrix u$, where $\lInertiaMatrix \dot \xi$ represents inertia and $\textrm{ad}_\xi\transpose \lInertiaMatrix \xi$ represents both the Coriolis and centripetal terms.
Configurations evolve according to the \textit{reconstruction equation} \cite{bloch2015nonholonomic}
\vspace{-2mm}\begin{equation}\label{eq:reconstruction}
    \dot \lGroupEle =\lGroupEle\xi^\wedge .\vspace{-2mm}
\end{equation}
Notice that $\lCoad$ makes the EP equations nonlinear in $\xi$ and that Eq. \eqref{eq:reconstruction} introduces coupling between the different dimensions. 
The system state in Lie group dynamics takes the form $s_k:=(g_k,\xi_k) \in \lGroup \times \lAlgebra$.
For systems subject to nonholonomic constraints, the twists evolve according to the  \textit{forced Euler-Poincar\'e-Suslov (EPS) equations} \cite{bloch2015nonholonomic}
\vspace{-2mm}\begin{equation} \label{eq:eps_motion}
    \frac{d}{dt}\frac{\partial l}{\partial \xi} = \lCoad\left(\frac{\partial l}{\partial \xi}\right) + \lInputMatrix u + \lConstMatrix\transpose \lambda, \quad \lConstMatrix\xi = 0,
\end{equation}
\looseness-1 where $\lConstMatrix \in \mathbb R^{k\times d}$ describes the $k$ nonholonomic constraints in body-fixed coordinates and $\lambda \in \mathbb R^k$ holds Lagrange multipliers. The configurations still evolve according to Eq. \eqref{eq:reconstruction}.

\vspace{-2mm}
\subsection{Split Conformal Prediction} \label{sec:split-cp}\vspace{-1mm}
We focus on Split Conformal Prediction (SplitCP) due to its computational efficiency and suggest \cite{angelopoulos2024theoretical} for a deeper understanding of the field. 
Let an unknown (stochastic) process map from input space $\CPmodelInputSpace$ to output space $\CPoutputSpace$.
Let $\CPmodel : \CPmodelInputSpace \to \CPmodelPredSpace$ be a fixed model --- analytical or learned, deterministic or stochastic --- mapping $\CPmodelInputSpace$ to prediction space $\CPmodelPredSpace$. 
Note that the prediction and output spaces are not necessarily the same, i.e., generally $\CPmodelPredSpace \neq \CPoutputSpace$.
SplitCP assumes access to a (calibration) dataset of process input and outputs $D_{cal}:=\{ (\CPmodelinput, \CPoutput) \}_{1:N}$, where $\CPmodelinput \in \CPmodelInputSpace, \CPoutput\in \CPoutputSpace$, and $(\CPmodelinput, \CPoutput)\sim \CPdist$ for an arbitrary unknown distribution $\CPdist$.
Then, we can construct a scalar-valued nonconformity score $r: \CPmodelPredSpace \times \CPoutputSpace \to \mathbb R$ measuring disagreement between model predictions and process outputs.
By applying $\CPmodel$ and $r$ to $D_{cal}$, we can construct a set of residuals $R_{cal}:=\textrm{sort}(\{r(\CPmodelpred, \CPoutput)\}_{1:N})$, where $\CPmodelpred:=\CPmodel(\CPmodelinput)$ and smaller values indicate greater model accuracy.
Then, given a new test tuple $(\CPmodelinput,\CPoutput)\sim \CPdist$, where $\CPmodelinput$ is known but $\CPoutput$ unobserved, SplitCP provides the following \textit{marginal coverage guarantee}
\vspace{-2mm}\begin{equation}
    \mathbb P\{r(\CPmodelpred, \CPoutput)\le \hat q_\alpha\} \ge (1-\alpha), \vspace{-2mm}
\end{equation}
where $\hat q_\alpha \in \mathbb R$ is the $\lceil(1-\alpha)(n+1)\rceil$ element of $R_{cal}$ and $\alpha \in (0,1)$ is the user-defined acceptable failure probability.
Note that the test tuple must be drawn from the same distribution $\CPdist$ as the calibration data, and the coverage guarantee holds non-asymptotically.
It directly follows that the prediction region $\CPpredRegion_\alpha$, defined by
\vspace{-2mm}\begin{equation}
    \CPpredRegion_\alpha(\CPmodelinput) := \{ \CPoutput \in \CPoutputSpace : r(\CPmodelpred, \CPoutput) \le \hat q_\alpha \}, \vspace{-2mm}
\end{equation}
\looseness-1 is marginally valid. Consequently, $\mathbb P(\CPoutput\in \CPpredRegion_\alpha(\CPmodelinput)) \ge (1-\alpha)$, where the probability is averaged over the test-time conditions (not for a specific $y$).
While the marginal guarantees are agnostic to the choice of $\CPmodel$ and $r$, they still impact how \textit{efficient} the prediction regions are, and hence SplitCP's practical utility.

\vspace{-1mm}
\section{CLAPS}\vspace{-2mm}

\looseness-1 We present our algorithm for constructing \textit{calibrated prediction regions} that provably contain the unknown robot configuration $q_1$ at the specified likelihood, on average over the test-time conditions, despite both epistemic and aleatoric uncertainty.
We also define a symmetry-respecting error metric for \SEtwo~ that can lead to increased region \textit{efficiency}.
Fig.~\ref{fig:clasps_model} provides an overview of the approach, and Alg.~\ref{alg:CLASPS} provides more detail.

\begin{figure*}[t!]
    \vspace{2mm}
  \centering
  \includegraphics[width=1.0\linewidth]{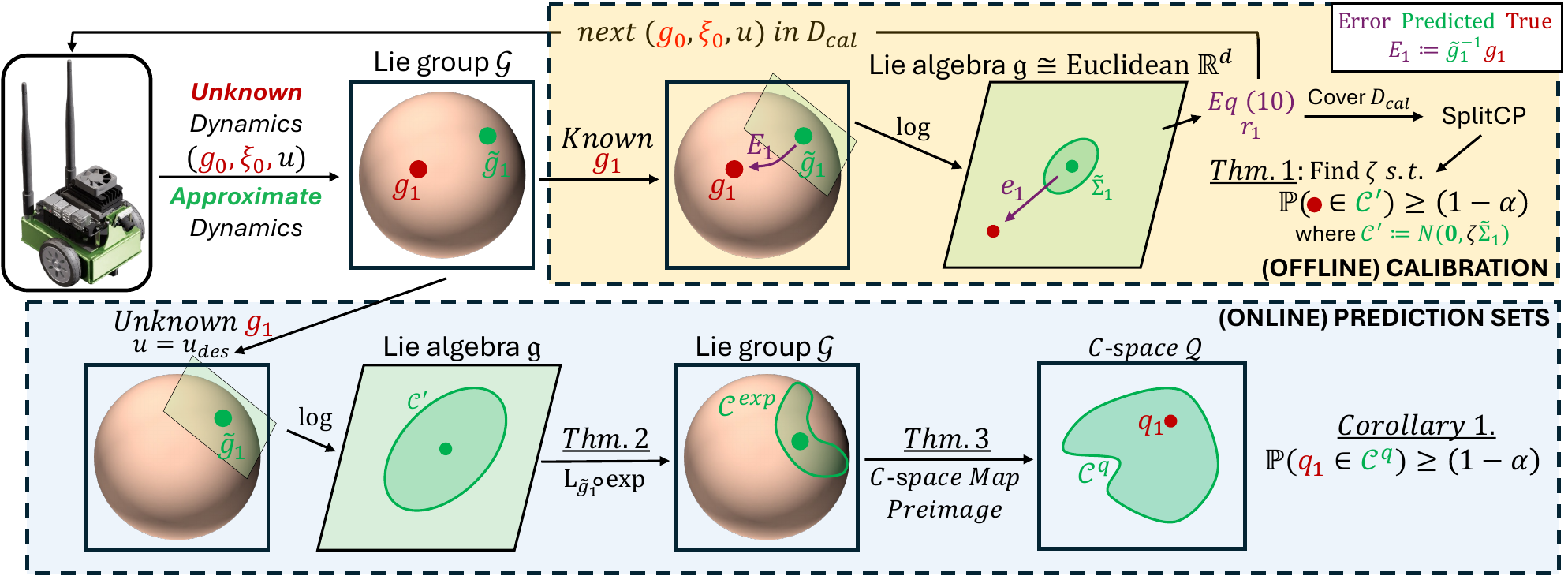}
  \vspace{-7mm}
  \caption{\longMethodName{}$\mid$ Offline: a dataset of state transitions is used jointly with an approximate dynamical model to derive a rigorous symmetry-aware probabilistic error bound on the configuration predictions. Online: our algorithm takes in a desired action $u_{des}$ and computes a \textit{calibrated C-space prediction region} $\CPpredRegionQ$ that is marginally guaranteed to contain the true configuration resulting from executing $u_{des}$.}
   \label{fig:clasps_model}
   \vspace{-5.5mm}
\end{figure*}

\vspace{-2mm}
\subsection{Converting State Space Dynamics to Lie group Form}\label{sec:convert_to_EPS}\vspace{-1mm}

Following the theoretical contributions of
\cite{bloch2009quasivelocities,bloch2015nonholonomic}
we provide the main equations for converting nonholonomic SS dynamics to Lie group form\footnote{\noindent 
Appendix~\ref{sec:APPconvert_to_EPS}
includes a complete, self-contained presentation of dynamics conversion from SS to Lie form.}. Let the \textit{kinematics map} $\FKmap:\mathcal Q \to \lGroup$ map generalized coordinates $q$ to elements $g \in \lGroup$.
For $SE(i)$, we have
\vspace{-4mm}\begin{equation}\label{eq:FK_map}
    g:=\FKmap(q)= \begin{bmatrix}
        R(q) & t(q)\\ 0 & 1
    \end{bmatrix} \in  \lGroup ,\vspace{-2mm}
\end{equation}
\looseness-1
where $t(q) \in \mathbb R^i$ is a translation vector, $R(q) \in \textrm{SO}(i)$ a rotation matrix, and $K\in \mathbb R^{(i+1)\times (i+1)}$.
We can then write the \textit{body-Jacobian} as 
\vspace{-2mm}\begin{equation}
    \lTwistMap(q) :=\left[\left(g^{-1}\frac{\partial K}{\partial q^1} \right)^\vee, \;\dots,\; \left(g^{-1}\frac{\partial K}{\partial q^n} \right)^\vee \right] \in \mathbb R^{d\times n}, \vspace{-2mm}
\end{equation} which enables the following velocity relationships 
\vspace{-2mm}\begin{equation}\label{eq:twist_maps_dq_xi}
    \xi  = \lTwistMap(q) \dot q, \quad \text{and}\quad \dot q = \lTwistMap(q)^\dagger \xi , \vspace{-2mm}
\end{equation}
where $(\cdot )^\dagger$ denotes the Moore-Penrose pseudoinverse.
Given SS Pfaffian velocity constraints $A(q)\dot q=0$, their Lie form becomes $\lConstMatrix(q) :=\constMatrix(q) \lTwistMap(q)^\dagger$.
Similarly, the body-frame inertia matrix is now $\lInertiaMatrix(q) := (\lTwistMap(q)^\dagger)\transpose M(q) (\lTwistMap(q)^\dagger)$ and the control input map is $\lInputMatrix(q) := (\lTwistMap(q)^\dagger)\transpose \inputMatrix(q)$.
This enables converting some symmetric nonholonomic systems from State Space form $(q,\dot q)$ to Lie group form $(g,\xi)$. 
In Appendices~\ref{sec:APP-trickSpeedUpComp} and~\ref{sec:App-numericalValidationDyn},
we present a numerical method to speed up the numerical integration of nonholonomic systems in Lie group form
and numerically validate $\S$\ref{sec:convert_to_EPS} on a \textit{second-order unicycle}.

\vspace{-3mm}
\setlength{\textfloatsep}{0pt}
\begin{algorithm}[t!] \label{alg:CLASPS}
\setlength{\algomargin}{0pt}

\caption{\methodName}
\SetInd{0.2em}{0.5em}
\SetAlgoLined
\DontPrintSemicolon

\KwIn{$\tilde{f}, D_{cal}, \alpha,u_{des}$}
{\footnotesize\tcc{(Offline Calibration) Once for all $u_{des}$}}
\For()
{$(s_0, u_0, s_{1})^{(i)} \in D_{cal}$}
{
$\tilde s_1^{},\tilde \Sigma_1 \leftarrow \approxDynamics(s_0, u_0)$ \footnotesize{\tcp*[f]{Approximate prediction}}
\\
    $r_1\leftarrow \sqrt{\log(\tilde g_1^{-1}g_1)\transpose \tilde \Sigma_1^{-1} \log(\tilde g_1^{-1}g_1)}$ \footnotesize{\tcp*[f]{Score, Eq \eqref{eq:r1_score}}}
}
 $\zeta \leftarrow$ {\footnotesize\texttt{SplitCP}($R_{cal}, \alpha$)} \footnotesize{\tcp*[f]{Conformal scaling factor}}\\ 
{\footnotesize\tcc{(Online Prediction) Once per each $u_{des}$}}
$\tilde s_1^{},\tilde \Sigma_1 \leftarrow \approxDynamics(s_0, u_{des})$ \footnotesize{\tcp*[f]{Approximate prediction}}\\
$\CPpredRegionPrime \leftarrow N(0,\zeta \tilde \Sigma_1)$ \footnotesize{\tcp*[f]{Calibrated Lie algebra Set}}\\

\end{algorithm}
\color{black}
\subsection{Formally Calibrating One-Step Action  Uncertainty}\label{sec:theorems_safety}
\vspace{-1mm}
\looseness-1
Given an initial state $s_0$ and a commanded action $u_{des}$, the available approximate model $\approxDynamics$ returns an expected next state $\tilde s_1$  and an uncertainty covariance $\tilde \Sigma_1$.
In this section, we derive a symmetry-informed and uncertainty-aware error metric $r_1$ between the prediction $\tilde g_1$ and the true unobserved resulting configuration $g_1$. Then, via SplitCP, we obtain a probabilistic upper bound $\hat q_\alpha$ for the test-time error. Subsequently, we construct a symmetry-respecting prediction region proved to marginally contain $q_1$ at the user-defined probability. \break
\indent \looseness-1 Existing work has treated robot states $s_k=(q_k,\dot q_k)$ as Euclidean, with both point predictions and uncertainty estimates (if existing) lying in SS.
Instead, we respect the natural symmetry of the robot and use \eqref{eq:FK_map} and \eqref{eq:twist_maps_dq_xi} to map $s_0$ to Lie group form $s_0=(g_0,\xi_0)$.
Then, given $u_{des}$, either by using a Lie group-$\approxDynamics$, or by using a State Space-$\approxDynamics$ and subsequently converting back to Lie form, we can obtain an expected next state  $(\tilde g_1, \tilde \xi_1):= \mathbb E[\tilde f(g_0,\xi_0,u_0)]$. Here, $\tilde g_1\in \lGroup$ is the expectation of the pose and $\tilde \xi_1 \in \lAlgebra$ is the expectation of the twist. 
Given the objective of building a region containing the true \textit{unknown} configuration, a natural error metric between the true Lie configuration $g_1$ and the expected prediction $\tilde g_1$ is the group difference $E_1:=\tilde g_1^{-1} g_1 \in \lGroup$. This represents the left-invariant displacement required to go from $\tilde g_1$ to $g_1$.
Yet, this is only a point-wise metric that does not account for the estimated uncertainty. 
\newline\indent
As \cite{marques2024quantifying} used an estimated state-space covariance $\tilde \Sigma_1$ to construct an uncertainty-aware error metric, we make a similar choice but instead place the uncertainty covariance in exponential coordinates $\mathbb R^d$, centered at the origin. 
Physically, our $\tilde \Sigma_1$ represents estimated body-fixed displacements and rotations caused by aleatoric disturbances.
Then, we can map $E_1$ to exponential coordinates through $e_1 := \log (E_1) \in \mathbb R^d$ and define the symmetry-based uncertainty-aware nonconformity score $r_1$ to be the Mahalanobis Distance between a configuration $g$ and the Gaussian Prediction $(\tilde g_1, \tilde \Sigma_1)$, i.e.,
\vspace{-2mm}\begin{equation}\label{eq:r1_score}
   r_1(g;\tilde g_1,\tilde \Sigma_1) = \sqrt{\log(\tilde g_1^{-1}g)\transpose \tilde \Sigma_1^{-1} \log(\tilde g_1^{-1}g)} .\vspace{-2mm}
\end{equation}
\looseness-1 As the exponential coordinates are a vector space, we can compute $r_1$ efficiently, while scores defined directly in the group could be more challenging to compute tractably.

\looseness-1
By looping over each item in $D_{cal}$ and calculating its nonconformity score $r_1$, we can build a set of scalars $R_{cal}$ and calculate $\hat q_\alpha$.
Applying the results of \cite{marques2024quantifying}, valid for Euclidean vector spaces, we can then \textit{calibrate the approximate uncertainty prediction} $N(0, \tilde \Sigma_1)$ in exponential coordinates, by scaling its covariance through the \textit{conformal scaling factor} $\zeta := \hat q_\alpha^2/\chi^2_\alpha(\dim\lAlgebra) \in \mathbb R$ into $N(0,\zeta \tilde \Sigma_1)$\footnote{$\chi^2_\alpha(\dim \lAlgebra)$ denotes the $(1-\alpha)$-quantile of the $\chi^2$ distribution of dimension $\dim \lAlgebra$. See \cite{marques2024quantifying} for more intuition into the \textit{scaling factor} $\zeta$.}. Then, the $(1-\alpha)$ confidence region of $N(0,\zeta \tilde \Sigma_1)$ will contain at least $100(1-\alpha)\%$ of the true unknown $e_1$. Formally, we have:
\begin{theorem}[Thm 2 of \cite{marques2024quantifying}] \label{thm:cp-xi-coverage} Let $e_1 \in \mathbb R^d$, $\zeta \in \mathbb R$ be defined as above, and let $\CPpredRegionPrime$ be the $100(1-\alpha)\%$ confidence region of $N(0, \zeta \tilde \Sigma_1)$. Then $\mathbb P(e_1\in \CPpredRegionPrime) \ge (1-\alpha)$.
\end{theorem}

\looseness-1
See Theorem 2 of \cite{marques2024quantifying} for the proof.
While this guarantees the marginal probabilistic containment of the true unknown error vector $e_1$ in exponential coordinates, no matter the fidelity of $\approxDynamics$ producing $\tilde \Sigma_1$, we aim to contain the true configuration $q_1$ in $\CPpredRegionQ$.
To bridge this gap, we consider how the distribution of $e_1$ propagates through the sequence of maps 
$\mathbb R^d \xrightarrow{(\cdot )^\wedge} \lAlgebra \xrightarrow{\exp_m} \lGroup \xrightarrow{L_{\tilde g_1}}\lGroup \xrightarrow{K^{-1}} \mathcal Q$.
While $\exp_m$ transports to group elements near the identity $\mathfrak e$, $e_1$ was defined relative to the expected configuration $\tilde g_1$.
Thus, to correctly recenter the distribution of configurations in $\lGroup$, we also apply a left-translation $L_{\tilde g_1}$, shifting the predicted region's center to be at $\tilde g_1 \in \mathcal G$.
Let us also analyze how $\CPpredRegionPrime$ can be transformed to C-space.

\looseness-1 The wedge operator is a \textit{diffeomorphism}, a continuously differentiable and bijective map with a differentiable inverse.
While $\exp_m$ is not globally injective, there exists an open neighborhood $U\subset \lAlgebra$ around the origin where $\exp_m {\mid_U}$ is diffeomorphic \cite{lee2010introduction,hall2013lie}\footnote{For \SEtwo, $U$ includes twists with angular component $\lvert \theta \rvert <\pi$ \cite{barfoot2024state}.
}.
Since every composition of diffeomorphisms is a diffeomorphism \cite{lee2010introduction}, then $\exp\mid_U$ is also a local diffeomorphism in $(U)^\vee$.
Further, as $\log_m$ is the inverse of $\exp_m {\mid_U}$ \cite{hall2013lie}, 
then by construction $e_1\in U^\vee$.  
Let $\phi:=L_{\tilde g_1}\circ \exp$. We can then relate exponential error vectors to true group configurations using $\phi(e_1) = L_{\tilde g_1} \circ \exp(e_1)=\tilde g_1  \exp(\log(\tilde g_1^{-1}g_1))=\tilde g_1  \tilde g_1^{-1}g_1=g_1$, where $\exp$ and $\log$ cancel only because they are mutually inverse where $e_1$ lives.
To map the prediction region, note that $\CPpredRegionPrime$ may extend beyond $U^\vee$ after scaling by $\zeta$.
To ensure it lies in the bijective domain, we can clip the region as $\CPpredRegion := \CPpredRegionPrime \cap U^\vee$.
This does not affect the results from Theorem \ref{thm:cp-xi-coverage}, as $e_1\in U^\vee$ by construction, which implies $\{e_1\in \CPpredRegionPrime\}\Leftrightarrow \{e_1\in \CPpredRegion\}$ and hence $\mathbb P(e_1 \in \CPpredRegionPrime) = \mathbb P(e_1 \in \CPpredRegion)$.
As the left-translation $L_{\tilde g_1}$ is globally diffeomorphic \cite{lee2010introduction}, 
 $\phi$ is still diffeomorphic. 
Let the mapped region be $\CPpredRegionExp := \phi(\CPpredRegion)$.
The \textit{preimage} of $\CPpredRegionExp$ is by definition $\phi^{-1}(\CPpredRegionExp):= \{\xi \in \mathbb R^d :\phi(\xi) \in \CPpredRegionExp \}$.
Then, from set theory, for an arbitrary set $\CPpredRegion$ and map $\phi$ we have the following inclusion relation: $\CPpredRegion\subseteq \phi^{-1}(\phi(\mathcal C))$ \cite{lee2010introduction}.
The equality occurs iff $\phi$ is bijective, so we have $\CPpredRegion=\phi^{-1}(\phi(\CPpredRegion))$
Again, $\phi^{-1}$ is the preimage, not the inverse (which generally might not exist).
Without considering $U$, we would obtain the less tight inequality, since many algebra elements map to the same group element (due to angular symmetry).
We can use these set relations to show:

\begin{theorem} \label{thm:cp-X-coverage}
Let $\phi,\ \CPpredRegionExp$ be defined as above and $\CPpredRegion \subseteq \mathbb R^d$. It follows that $\mathbb P(\lGroupEle_1 \in \CPpredRegionExp) = \mathbb P(e_1 \in \CPpredRegion)$.
\end{theorem}
\begin{proof}
\looseness-1 Since $g_1 =\phi(e_1)$ and $\phi$ is bijective, the definition of  preimage gives $\{ e_1\in \CPpredRegion \}\Leftrightarrow \{ e_1\in \phi^{-1}(\CPpredRegionExp)\}\Leftrightarrow \{ \phi(e_1)\in \CPpredRegionExp\}=\{g_1 \in \CPpredRegionExp \}$, and the claim follows.
\end{proof}

By using the preimage, we placed no further requirements on $\exp$ or $\phi$, such as global invertibility or differentiability\footnote{Other works split $\exp$ into multiple diffeomorphic regions to construct a pushforward probability density in the group \cite{falorsi2019reparameterizing}. However, we only require set inclusion for our claims about containing sufficient probability mass.}.
Finally, the \textit{preimage} of the kinematics map, $K^{-1}$, transports from $\lGroup$ to $\mathcal Q$. We can thus obtain:
\begin{theorem}\label{thm:cspace_safe} Let $K, \CPpredRegionExp$ be defined as above and $g_1:= K(q_1)$. For $\CPpredRegionExp \subseteq \lGroup$, its preimage is $\CPpredRegionQ:= K^{-1}(\CPpredRegionExp)$. Then $\mathbb P(q_1 \in \CPpredRegionQ) = P(g_1 \in \CPpredRegionExp)$.
\end{theorem}
\begin{proof}
Using the definition of preimage, we get $\{q_1 \in K^{-1}(\CPpredRegionExp) \}=\{K(q_1) \in \CPpredRegionExp \}$, and the claim follows.
\end{proof}

Joining the results from the Theorems \ref{thm:cp-xi-coverage}, \ref{thm:cp-X-coverage}, \ref{thm:cspace_safe} we get:
\begin{corollary} \label{col:all_map}$\mathbb P(q_1 \in \CPpredRegionQ) \ge  (1-\alpha)$.
\end{corollary}
\looseness-1

We have shown that by calibrating an approximate uncertainty estimate in exponential coordinates and mapping it to a prediction region in C-space $\CPpredRegionQ$, our algorithm produces a set that marginally contains the unknown true configuration $q_1$ with at least the user-defined probability $(1-\alpha)$.
\vspace{-3mm}
\subsection{Example Downstream Applications}\label{sec:imp_tricks}\vspace{-1.5mm}
\looseness-1 After following Alg. \ref{alg:CLASPS}, we obtain a calibrated set $\CPpredRegionPrime \subseteq \lAlgebra$, which might be used for safe control in a few ways: A) Checking if a given configuration $g$ is contained in $\CPpredRegionPrime$ is equivalent to verifying the inequality $r_1(g;\tilde g_1,\zeta \tilde \Sigma_1) \le \chi^2_\alpha(\dim \lAlgebra)$. This can be done efficiently in a batched manner for thousands of points.
B) Sometimes, it is helpful to reconstruct the C-space set $\CPpredRegionQ$, for example, to check if $\CPpredRegionQ \subseteq \SAFE$ for a known set $\SAFE \subseteq \mathcal Q$, which would enable probabilistically safe one-step control.
This may be computationally expensive, however, the process can be simplified if $K$ has a diffeomorphic inverse $K^{inv}$, which occurs for \SEtwo~when the angle is restricted to a $2\pi$ interval.
Then, its composition with $\phi$, $K^{inv}\circ \phi$, is also diffeomorphic and can be used to map the exponential coordinates boundary $\partial \CPpredRegion$ to the C-space boundary $\partial \CPpredRegionQ$.
\begin{lemma}[Theorem 2.18 of \cite{lee2010introduction}]\looseness-1
Let $\partial \CPpredRegion, \partial \CPpredRegionQ$ denote the boundaries of $\CPpredRegion \subseteq \mathbb R^d$, $\CPpredRegionQ:= K^{inv}\circ \phi (\CPpredRegion) \subseteq \mathcal Q$ respectively. If $K^{inv}\circ \phi$ is diffeomorphic, $\partial \CPpredRegionQ =  (K^{inv}\circ \phi)(\partial \CPpredRegion)$. \label{lem:boundary_pt}
\end{lemma}

\looseness-1 We describe $\CPpredRegionQ$ reconstruction in Alg \ref{alg:fit-mesh}, use it in $\S$\ref{sec:isaac_exp} for calculating C-space volumes and empirical coverages, and provide an implementation tested to reconstruct $\CPpredRegionQ$ at up to 25 Hz.
C) Other times, safety requires performing intersection checks between a workspace $(\mathbb R^2)$ footprint and known obstacles. For \SEtwo, one can inflate obstacles by the robot's radius and subsequently treat the robot as a point. After sampling points in $\partial \CPpredRegionQ$ per B), map them to $\mathbb R^2$ by marginalizing the heading $\theta$. Then, either check the safety of these points directly, or reconstruct a 2D surface and check that instead. We provide code (Footnote \ref{code_url}) to reconstruct said 2D surface, which we used for Fig. \ref{fig:isaac_2d_marg}.

\vspace{-3mm}
\begin{algorithm}[h!] \label{alg:fit-mesh}
\setlength{\algomargin}{0pt}

\caption{Reconstruct C-space Mesh}
\SetInd{0.2em}{0.5em}
\SetAlgoLined
\DontPrintSemicolon
\footnotesize{\tcc{Sample points from $d$-dim unit sphere $\mathbb S^d$, map to Calibrated Lie algebra Set}}
$\partial \CPpredRegionPrime  \leftarrow$ $\sqrt{\chi^2_\alpha(d)} (\zeta \tilde \Sigma_1)^{1/2}   \texttt{Sample($\mathbb S^d$)}$ \\
$\partial \mathcal C  \leftarrow \partial  \CPpredRegionPrime \cap\ U^\vee$ \footnotesize{\tcp*[f]{Restrict to diffeomorphic}}\\
$\partial \CPpredRegionQ \leftarrow (K^{inv}\circ \phi)(\partial \CPpredRegion)$ \footnotesize{\tcp*[f]{Map points to C-space}}\\

$\CPpredRegionQ \leftarrow \texttt{Reconstruct}(\partial \CPpredRegionQ)$ \hspace{-3.2mm}\footnotesize{\tcp*[f]{Get mesh from points}}
\end{algorithm}\vspace{-6mm}
\vspace{-1.5mm}
\section{EXPERIMENTS \& DISCUSSION}
\label{sec:isaac_exp}
\vspace{-2mm}

\looseness-1 To support Theorems \ref{thm:cp-xi-coverage}-\ref{thm:cspace_safe} and Corollary \ref{col:all_map}, we conducted an empirical experiment on Isaac Sim using the JetBot (Fig. \ref{fig:clasps_model}) and a hardware experiment on the MBot platform (Fig. \ref{fig:title_figure}). These are both 
\textit{nonlinear}, \textit{underactuated}, and \textit{nonholonomic systems} that we modeled as \textit{second-order unicycles} with configuration in \SEtwo. See
Appendix~\ref{sec:APPconvert_to_EPS}
for the dynamics equations used and an application of $\S$\ref{sec:convert_to_EPS}. The true dynamics are unknown, the inertial properties are estimated with standard system identification, and both systems are subject to \textit{aleatoric disturbances}\footnote{\looseness-1 In
Appendix~\ref{sec:APP-systemIdentification},
we detail the inertial property estimation procedure for both systems, and potential sources of uncertainty.} --- in simulation, disturbances were injected following the Problem formulation $\S$\ref{sec:prob_statement}, and in hardware, disturbances are inherent to real robotic control (e.g., uneven terrain, wheel roughness, etc.).
Besides the errors introduced by the system identification, \textit{epistemic uncertainty} is also present due to effects not modeled by $\tilde f$, such as friction, CoM deviation from the body-fixed origin, actuation delay, and the conversion of desired body-frame forces/torques to commanded wheel torques. 

 \looseness-1 We compare \methodName~with seven baselines, to demonstrate the improved \textit{efficiency} of its prediction region $\CPpredRegionQ$ and \methodName'~ability to represent the underlying uncertainty. All methods are based on the prediction step of the Extended Kalman Filter (EKF), performed at $60$ Hz with Forward Euler discretization, and share the same inertia matrix estimate $\tilde \lInertiaMatrix$ and the same \textit{uncalibrated} initial uncertainty estimate $\tilde Q_0$. Hence the expected future pose $\tilde g_1$ (or $\tilde q_1$) is shared across approaches, which differ in how they represent and calibrate uncertainty. All results show prediction regions and Monte Carlo (MC) particles after one full planning step of $\dt=0.5$ sec. SS EKF performs the prediction step using State-space dynamics, resulting in ellipsoidal $\CPpredRegionQ$. The Invariant EKF (InEKF) \cite{barrau2016invariant,barrau2018invariant} uses Lie group dynamics, propagating a Gaussian uncertainty on the Lie algebra and leading to banana-shape regions \cite{long2013banana}. 
 These baselines do not consider $D_{cal}$. InEKF+2M uses the uncentered second moment of the one-step configuration errors $e_1$ in $D_{cal}$ as its uncertainty estimate, i.e., $\tilde Q_1^{2M} \approx \mathbb E[e_1 e_1\transpose]$. InEKF+MLE fits both a bias correction and a centered covariance to the $e_1$ in $D_{cal}$, i.e. $\tilde b_1 \approx \mathbb E[e_1], \tilde Q_1^{MLE}\approx\mathbb E[(e_1-\tilde b_1)(e_1-\tilde b_1)\transpose]$. These methods perform a data-driven estimation of the uncertainty covariance, yet they do \textit{not} adapt $\tilde Q_1$ to the commanded action. \textit{None} of the four methods above provide guarantees on $\CPpredRegionQ$ containing the future system configuration, thus being unsuitable for safety-critical control. 
  SS PP + CP is a common approach \cite{lindemann2023safe,dixit2023adaptive,sun2023conformal,strawn2023conformal} using the expectation of SS EKF's prediction as a point-prediction (PP) $\tilde q_1$ and performing conformal calibration by using the L2 distance between $\tilde q_1$ and the true resulting configuration $q_1$ as the score, i.e., $r_1=\lVert \tilde q_1 -q_1 \rVert$. Lie PP + CP is a naive extension of SS PP + CP to Lie groups. It uses Lie group dynamics and calculates the L2 distance in the Lie algebra, $r_1 = \lVert e_1 \rVert =\lVert \log (\tilde g_1^{-1}g_1)\lVert$. This results in a ball-shaped region that is then mapped from Lie algebra to C-space following the same map as \methodName. SS EKF + CP \cite{marques2024quantifying} performs uncertainty-aware calibration in a Euclidean C-space, using the Mahalanobis distance $r_1 = \sqrt{(q_1-\tilde q_1)\transpose \tilde \Sigma_1 (q_1-\tilde q_1)}$ as the score. Our proposed approach can be interpreted as a \textit{provably-correct symmetry-aware calibration of InEKF}.
 We use $\alpha=0.1$ in all experiments. 
 \vspace{-2.5mm}
\subsection{Simulation Experiments (JetBot)}\label{sec:sim_exp_only}\vspace{-1.5mm}
\looseness-1 We independently sample $u_{noise} \sim Q_{cont} :=N(0,\textrm{diag}(0.005, 0.001))$ and additively perturb the desired wrenches to represent aleatoric external wrench disturbances, i.e., $u_{cmd} = u_{des} + u_{noise}$. While $u$ is a wrench, we write the corresponding accelerations for interpretability.
Let $\texttt{lin}(a,b,N)$ denote a linearly spaced sequence of $N$ reals from $a$ to $b$. 
The calibration dataset $D_{cal}$ was collected by spanning the grid:
$s_0=(0,0,0)$; $\dot x_0^b\in \texttt{lin}(0.1, 0.5, 3);\dot y_0^b=0;\dot \theta_0^b\in \texttt{lin}(0, 0.5, 3);\ddot x^b\in\texttt{lin}(0, 0.5,3);\ddot \theta^b \in \texttt{lin}(0, 2, 3)$.
To capture the aleatoric effects, we sampled $500$ transitions per gridpoint, totaling $\lvert D_{cal}\rvert =40,500$.
The validation set spanned the same grid, with now $5$ partitions per interval,
resulting in $5^4=625$ cases.
Thus, the validation and calibration data share the same distribution (Assumption \ref{ass:ex} holds).
For each validation case, $100k$ JetBots were propagated in Isaac, each with their own independently sampled $u_{noise}$.
These MC particles enable the calculation of \textit{empirical coverage}, i.e., the probability that a system under the true unknown stochastic dynamics produces configurations in the prediction region $\CPpredRegionQ$. By averaging empirical coverage over the $625$ trials, we obtain an estimate of the \textit{marginal coverage} provided by each algorithm.
Additionally, we reconstructed $\CPpredRegionQ$ using Alg. \ref{alg:fit-mesh} with $5k$ samples, to check for containment of the MC particles and to compute C-space volumes.
One trial can be seen in Figure \ref{fig:val_isaac_case}, with $\partial \CPpredRegionQ$ for each method in orange. Following B) of $\S$\ref{sec:imp_tricks}, these C-space geometries could be used for safe control. 
\begin{figure}[b!]
  \centering
\includegraphics[width=\linewidth]{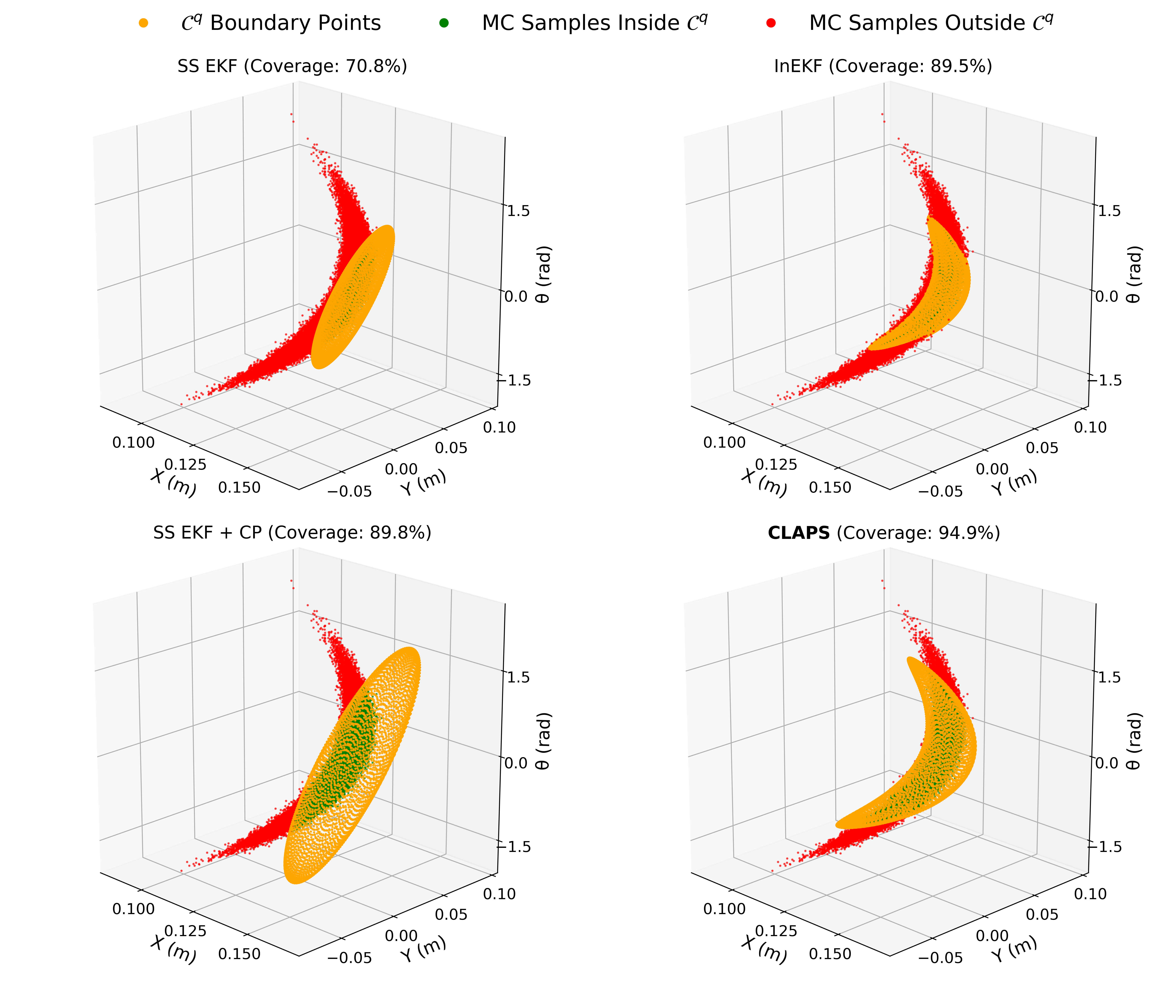}
  \vspace{-9mm}
  \caption{\looseness-1 C-space prediction regions for one of the $625$ simulation trials, with 100k MC particles overlaid representing the system's stochasticity.
  Each method's empirical coverage for this trial is shown ($\alpha=0.1$).
  SS EKF and InEKF are not guaranteed to contain the future configurations at the user-set likelihood. The ball-shaped Point Prediction (PP) baselines produced regions too large to plot in the same scale.
  Since SS EKF + CP treats configurations as Euclidean vectors, its regions are restricted to hyper-ellipsoids that do not capture the underlying uncertainty as well as \methodName. }
   \label{fig:val_isaac_case}
   \vspace{-5mm}
\end{figure}
\break\indent
 To further demonstrate \methodName' intuitive bounds, Fig. \ref{fig:isaac_2d_marg} shows the workspace footprint of $\CPpredRegionQ$, which could be used for probabilistic obstacle avoidance. Both figures qualitatively demonstrate \methodName~ability to fit the underlying system uncertainty (represented by MC samples). This is supported quantitatively by \methodName' larger Workspace Intersection-over-Union (IoU) with the MC samples. We report Marginal Coverage, Relative C-space Volume (averaged over trials), and Workspace IoU with MC Samples in Table 1.
\begin{figure}[t!]
  \centering
\includegraphics[width=\linewidth]{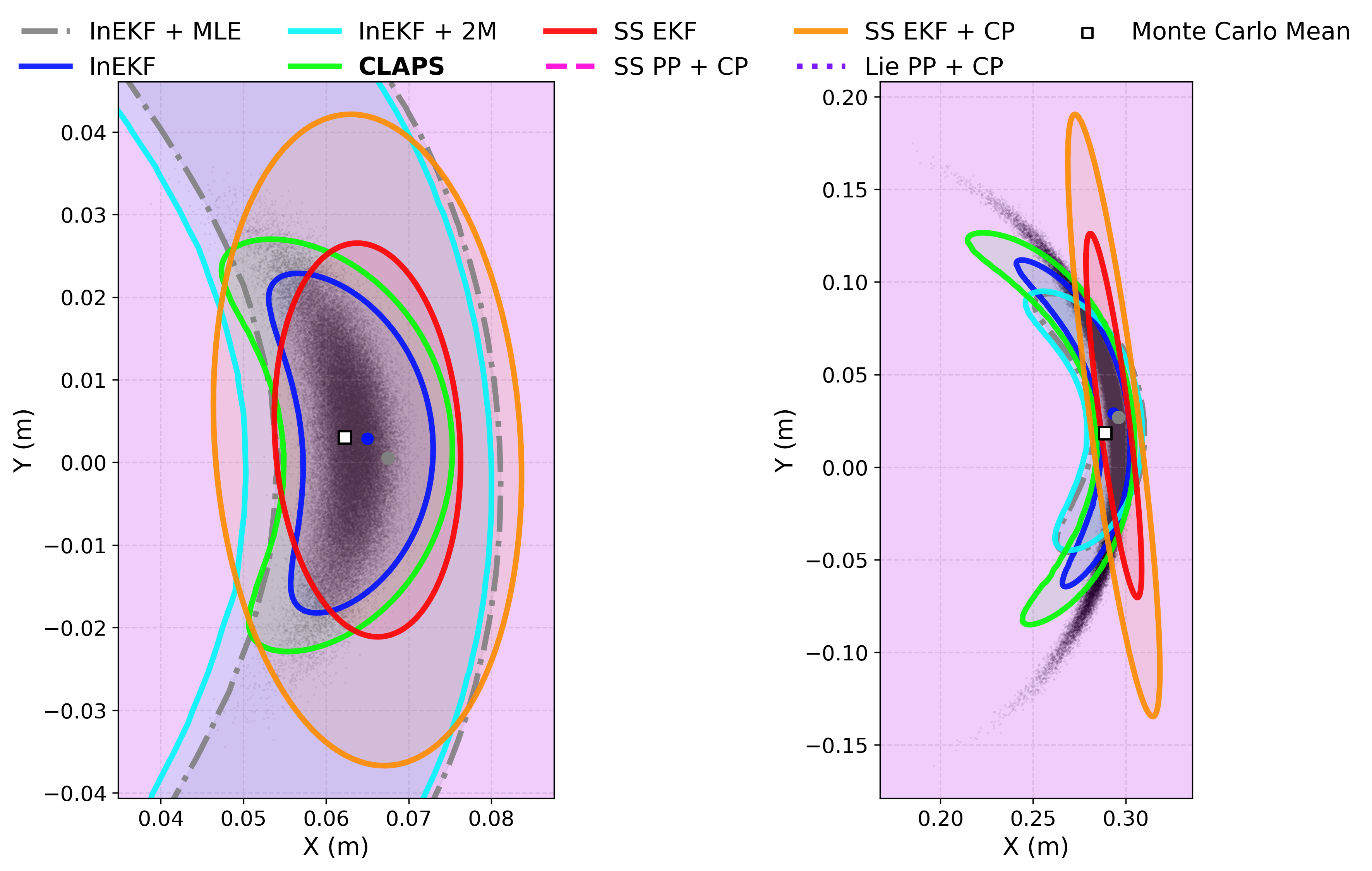}
  \vspace{-9mm}
  \caption{\looseness-1 Workspace ($\mathbb R^2$) Marginalization of the C-space regions generated by all the methods, over two JetBot trials. 
  Left: $\dot x_0^b=0.1;\dot \theta^b_0=0; \ddot x^b_0=0.35;\ddot \theta^b_0=0.007$. Right: $\dot x_0^b=0.5;\dot \theta^b_0=0.375; \ddot x^b_0=1.05;\ddot \theta^b_0=0$.
  InEKF+MLE has expected pose $\tilde g_1$ shown as the gray dot. All other methods have the same expected pose, which is represented by the blue dot. The Point Prediction (PP) methods generate large regions with boundaries lying outside the plots' margins. SS EKF, InEKF, InEKF+2M, and InEKF+MLE are not guaranteed to contain the resulting configuration at the user-set likelihood. Qualitatively, \methodName~appears to more accurately represent the underlying uncertainty distribution than the symmetry-unaware baselines. \label{fig:isaac_2d_marg}}
\end{figure} 
\begin{table}[b]
\vspace{1mm}
\captionsetup{              
  skip   = 4pt,             
}
\centering
\scriptsize
\setlength{\tabcolsep}{2pt}
{\setstretch{0.8}\caption{JetBot (Simulation) Results (over $625$ validation trials)}}
\label{tab:isaac_results}

\begin{tabular}{l|c|c|c|c}
\toprule
\multirow{2}{*}{Algorithm}
  & Marginal
  & Avg.
  & \multicolumn{1}{c|}{Avg. Workspace IoU} & \multicolumn{1}{c}{Provable}\\
& Coverage $(\%)$& Volume Ratio $\downarrow$
  & \multicolumn{1}{c|}{with Particles $(\%)\uparrow$} & \multicolumn{1}{c}{Guarantees?}\\

\midrule
SS EKF & \ccell{red!15}78.7  & \ccell{red!15}0.63 & \ccell{red!15}35.0 & \myxmark \\ 
\midrule
InEKF & \ccell{red!15}82.7 & \ccell{red!15}0.47 & \ccell{red!15}41.8 & \myxmark \\ 
\midrule
InEKF+2M & \ccell{red!15}89.2 & \ccell{red!15}3.06 & \ccell{red!15}40.0 & \myxmark\\ 
\midrule
InEKF+MLE & 90.3 & 2.80 & 42.3 & \myxmark\\ 
\midrule
SS PP + CP & 89.9  & 2137 & 0.20 & \mycheck \\ 
\midrule
Lie PP + CP & 89.9  & 2138 & 0.20 & \mycheck \\ 
\midrule
SS EKF + CP & 91.2  & 2.86 & 30.4 & \mycheck \\ 
\midrule
\methodName & 90.0  & \textbf{1.00} & \textbf{48.4} & \mycheck  \\ 
\midrule
\end{tabular}
\vspace{1pt}
\begin{minipage}{\linewidth}
\raggedright\scriptsize
\colorbox{red!15}{red} if coverage does not achieve $(<)$ the user-set probability $(1-\alpha)=0.9$.\\
The average volume ratio is reported relative to \methodName.
\end{minipage}
\end{table}
\break\indent
\looseness-1 The \textit{uncalibrated} SS EKF and InEKF fail to satisfy the user-set specification. InEKF+2M and InEKF+MLE estimate the same uncertainty for all initial velocities and $u_{des}$, thus becoming volume inefficient. The CP methods achieve at least $(1-\alpha)\%$ coverage\footnote{Up to numerical error based on the limited number of validation trials.}, supporting Corollary \ref{col:all_map}. Since $\CPpredRegionQ$'s volume increases non-monotonically with $\alpha$ for each method, i.e., $\mathcal C_{\alpha_1} \subseteq \mathcal C_{\alpha_2} \Rightarrow \alpha_1 \ge \alpha_2$, higher marginal coverage is not necessarily beneficial, and all methods not in red should be treated as equally \textit{calibrated}. Both algorithms using L2-based scores (SS PP + CP, Lie PP + CP) construct significantly large ball-shaped $\CPpredRegionQ$, making them impractical.
 Additionally, symmetry-unaware methods (e.g., SS EKF + CP) can create $\CPpredRegionQ$ covering volumes of zero support (e.g., contain volumes where $\lvert \theta\rvert  > \pi)$.
Our method produces efficient banana shaped regions containing a satisfactory probability mass of the future configurations --
\methodName' $\CPpredRegionQ$ has smaller C-space volumes than all calibrated baselines in the 625 validation trials we tested. Further, \methodName~achieves the highest average IoU with the MC Particles, validating \methodName' better representation of the underlying uncertainty. Compared with other CP methods, \methodName~achieves a higher IoU in each tested trial. 
The trials also supported Theorems 1-3, with
the MC particles satisfying $e_1 \in \CPpredRegion \Rightarrow g_1 \in \CPpredRegionQ$.

\vspace{-2.5mm}
\subsection{Hardware Experiments (MBot)}\label{sec:real_experiment}\vspace{-1mm}
The robot's pose and velocity were estimated using Motion Capture (Fig. \ref{fig:title_figure}). Calibration and validation data were collected by randomly sampling $u_{des}$ from $\ddot x^b \in (0, 0.5); \ddot \theta^b \in (0, 2)$ and holding it for $\dt$. The MBot's velocity was kept to approximately within $\dot x^b \in (0.1, 0.3); \lvert \dot \theta^b\rvert \in (0, 0.5)$. We shuffled the data, allocating $5\%$ for calibration ($\rvert D_{cal}\lvert=237$), which corresponds to $\approx2$ min of driving, and leaving $4511$ transitions for validation. Since there is a single \say{MC Particle} per transition, we cannot compute IoU as in simulation. Table 2 shows the coverage and C-space volumes.
\begin{table}[t]
\captionsetup{skip=4pt}
\centering
\scriptsize
\setlength{\tabcolsep}{2pt} 
{\setstretch{0.8}\caption{MBot (Hardware) Coverage, Volume ($4511$ validation trials)}}
\label{tab:real_mbot_results}
\begin{tabular}{l|c|c|c|c|c|c|c|c}
\toprule
Algorithm & \makecell{SS\\EKF} & \makecell{InEKF} & \makecell{InEKF\\ +2M} & \makecell{InEKF\\ +MLE} & \makecell{SS PP\\+ CP} & \makecell{Lie PP\\+ CP} & \makecell{SS EKF\\+ CP} & \methodName \\
\hline
\makecell{Marginal\\ Coverage (\%)} 
  & \ccell{red!15}{73.5}
  & \cellcolor{red!15}70.6
  & \cellcolor{red!15}87.4
  & \cellcolor{red!15}86.9
  & 90.5
  & 90.5
  & 91.8
  & 90.4 \\
  \midrule
   \makecell{Avg. Volume \\ Ratio $\downarrow$} & \ccell{red!15}0.32 &\ccell{red!15}0.27  &\ccell{red!15}0.08 &\ccell{red!15}0.05 & 2.09 &2.09 & 1.30 & \textbf{1.0}\\
   \midrule
   \makecell{Provable \\ Guarantees?} & \myxmark & \myxmark& \myxmark & \myxmark & \mycheck & \mycheck &\mycheck & \mycheck\\
\bottomrule
\end{tabular}

\vspace{1pt}
\begin{minipage}{\linewidth}
\raggedright\scriptsize
\colorbox{red!15}{red} if coverage does not achieve $(<)$ the user-set probability $(1-\alpha)=0.9$.\\
The average volume ratio is reported relative to \methodName.
\end{minipage}
\end{table}

\looseness-1 In this low-data experiment, the approximate estimators failed to achieve the user-set requirement, while the CP methods satisfied it -- as expected from the \textit{finite-sample guarantees}. We observed less angular uncertainty with the MBot than in $\S$\ref{sec:sim_exp_only}, possibly due to larger ground friction and slower speeds.
This partially explains the volume ratio reduction, as the MC particles may be reasonably captured by convex regions in low-uncertainty regimes, while symmetry-awareness becomes more important as angular uncertainty grows \cite{long2013banana}. Still,
\methodName~produced a smaller average $\CPpredRegionQ$ than all calibrated baselines,
demonstrating its volume efficiency in real situations. Compared to SS EKF + CP, \methodName' regions were on average 23\% smaller and up to 75\% smaller.
The Offline step of \methodName~(lines 1-5 in Alg \ref{alg:CLASPS}) took $0.14$ sec on an Intel i9-12900K, and the Online portion (lines 6-7) took $0.02$ sec per $u_{des}$. The $\CPpredRegionQ$ reconstruction (Alg. \ref{alg:fit-mesh}) took $0.35\pm0.02$ sec with $5k$ boundary particles and $0.04$ sec with $500$ particles
-- acceptable given the MBot's sampling rate of 25 Hz. 
In Appendix~\ref{sec:APP-sensitivityNrParticles},
we discuss the sensitivity of our mesh reconstruction to the number of sampled boundary points.
Per $\S$\ref{sec:imp_tricks}, Alg. \ref{alg:fit-mesh} is not required for downstream use for safe control, as collision checks can be performed in the workspace or on individual sampled points.

\vspace{-2mm}
\color{black}
\section{CONCLUSION}\vspace{-2mm}

\looseness-1 We proposed an algorithm that enables constructing \textit{calibrated prediction regions} when under both \textit{aleatoric} and \textit{epistemic uncertainty}.
Our method leverages the robot's symmetry to construct regions that appear to be more volume-efficient and a better representation of the underlying uncertainty than existing approaches, both in simulation and hardware, extending previous CP guarantees from Euclidean space to robots with configurations in \SEtwo.

\bibliographystyle{IEEEtran}
\bibliography{references}

\clearpage
\onecolumn
\appendices
\section*{APPENDIX}
\renewcommand{\thesubsection}{\Alph{subsection}}

\looseness-1 This extended section provides a clear and self-contained presentation of the contributions of \cite{bloch2009quasivelocities,bloch2015nonholonomic}. We also include more details about the experiments and the open-source code, a numerical validation of converting dynamics from SS to Lie group form, and an equation to avoid computing EPS Lagrange multipliers at every integration step.

\subsection{Converting State Space Dynamics to Lie group Form}\label{sec:APPconvert_to_EPS}
Following the theoretical contributions on
the momentum map and the concept of quasivelocities
\cite{bloch2009quasivelocities,bloch2015nonholonomic},
we provide below a self-contained presentation of the process for converting SS nonholonomic dynamics to Lie form. Fig. \ref{fig:APPapp-eq_relation} illustrates the relationships between these two forms.

Let the \textit{kinematics map} $\FKmap:\mathcal Q \to \lGroup$ map generalized coordinates $q$ to elements $g \in \lGroup$.
For $SE(i)$, one can write
\begin{equation}\label{eq:APPFK_map}
    g:=\FKmap(q)= \begin{bmatrix}
        R(q) & t(q)\\ 0 & 1
    \end{bmatrix} \in  \lGroup ,
\end{equation}
\looseness-1
where $t(q) \in \mathbb R^i$ is a translation vector, $R(q) \in \textrm{SO}(i)$ a rotation matrix, and $K\in \mathbb R^{(i+1)\times (i+1)}$.
To relate generalized velocities $\dot q$ to body-frame twists $\xi^\wedge$, we start by re-arranging the \textit{reconstruction eq.} \eqref{eq:reconstruction} into
\begin{equation}\label{eq:APPderive_jacob_step}
    \xi^\wedge =g^{-1}\dot g =g^{-1}\sum_{j=1}^n
\frac{\partial K(q)}{\partial q^j} \dot q^j =\sum_{j=1}^ng^{-1}
\frac{\partial K(q)}{\partial q^j} \dot q^j ,
\end{equation}
where $q^j$ denotes the $j$-th dimension of $q$, and the summation resulted from applying the chain rule to $\frac{d}{dt}g=\frac{d}{dt}K(q)$.
The inverse $g^{-1}$ exists by definition for any Lie group and in our case takes the form $g^{-1}=\left[\begin{smallmatrix}
R\transpose & -R\transpose t \\ 0 & 1
\end{smallmatrix}\right]$.
Each $\frac{\partial K(q)}{\partial q^j}$ is a matrix-valued element of the tangent space $T_{g}\lGroup$ since it describes an infinitesimal change in $g$ due to $q^j$. Left multiplying by $g^{-1}$ transports this tangent vector from $T_{g}\lGroup$ to $T_{\mathfrak e}\lGroup$, so that $g^{-1}({\partial K(q)}/{\partial q^j})\in \lAlgebra$.
Since $\dot q^j$ is a scalar, we can apply the vee map (linear on $\lAlgebra$) to both sides of \eqref{eq:APPderive_jacob_step} to get
\begin{equation*}
    (\xi^\wedge)^\vee =\xi =\left(\sum_{j=1}^ng^{-1}
\frac{\partial K(q)}{\partial q^j} \dot q^j\right)^\vee = \sum_{j=1}^n \left(g^{-1}
\frac{\partial K(q)}{\partial q^j} \right)^\vee \dot q^j .
\end{equation*}
We then collect the terms $\left(g^{-1}\frac{\partial K(q)}{\partial q^j}\right)^\vee \in \mathbb R^d$ column-wise into
\begin{equation}
    \lTwistMap(q) :=\left[\left(g^{-1}\frac{\partial K}{\partial q^1} \right)^\vee, \;\dots,\; \left(g^{-1}\frac{\partial K}{\partial q^n} \right)^\vee \right] \in \mathbb R^{d\times n}, 
\end{equation} the \textit{body-Jacobian}. This enables the velocity relationships 
\begin{equation}
    \label{eqAPP:twist_maps_dq_xi}
    \xi  = \lTwistMap(q) \dot q, \quad \text{and}\quad \dot q = \lTwistMap(q)^\dagger \xi ,
\end{equation} 
where $(\cdot )^\dagger$ denotes the Moore-Penrose pseudoinverse.
We assume the user has chosen $q$ s.t. $K$ is $C^1$ and the body-Jacobian has full rank for all $q\in \mathcal Q$.
A full-rank ensures the mapping $\dot q \mapsto \xi$ is surjective and that the mapping $\xi \mapsto \dot q$ produces an exact (but generally not unique) solution. For \SEthree, quaternions (or appropriate alternatives) should be used instead of Euler angles, which lose rank at gimbal-lock configurations.
We now show how to convert velocity constraints, inertial properties, and input maps.
Given Pfaffian velocity constraints $A(q)\dot q=0$, we substitute \eqref{eqAPP:twist_maps_dq_xi} to get $A(q)\dot q=A(q) \lTwistMap(q)^\dagger \xi$. Then to reach the EPS-compatible form $\lConstMatrix \xi = 0$, we define body-fixed constraint matrix $\lConstMatrix$ as
\begin{equation}
    \lConstMatrix(q) :=\constMatrix(q) \lTwistMap(q)^\dagger. 
\end{equation}
To obtain the body-frame inertia matrix from the generalized inertia matrix $M(q)$, we note that kinetic energy must be independent of the representation.
Hence, $T(q,\dot q)=\frac{1}{2}\dot q\transpose M(q)\dot q = \frac{1}{2}\xi \transpose\lInertiaMatrix \xi = l(\xi)$, and substituting \eqref{eqAPP:twist_maps_dq_xi} yields $\lInertiaMatrix(q) := (\lTwistMap(q)^\dagger)\transpose M(q) (\lTwistMap(q)^\dagger)$.
To convert external generalized forces to body-frame forces, we preserve mechanical power which gives $(\inputMatrix(q)u)\transpose \dot q= (\lInputMatrix(q) u)\transpose \xi$. Substituting again $\dot q=\lTwistMap(q)^\dagger \xi$ into the left-side gives $u\transpose B(q)\transpose \lTwistMap(q)^\dagger \xi = u\transpose \lInputMatrix(q)\transpose \xi \Rightarrow B(q)\transpose \lTwistMap(q)^\dagger = \lInputMatrix(q)\transpose$, which ultimately leads to the relationship 
\begin{equation}
    \lInputMatrix(q) := (\lTwistMap(q)^\dagger)\transpose \inputMatrix(q). 
\end{equation}
A similar procedure could be used to convert other forces between generalized and body-fixed frames.
With these mappings, one can convert most symmetric nonholonomic systems from State Space form $(q,\dot q)$ to Lie group form $(g,\xi)$.

\begin{figure}[t]
  \centering
  \includegraphics[width=.7\linewidth]{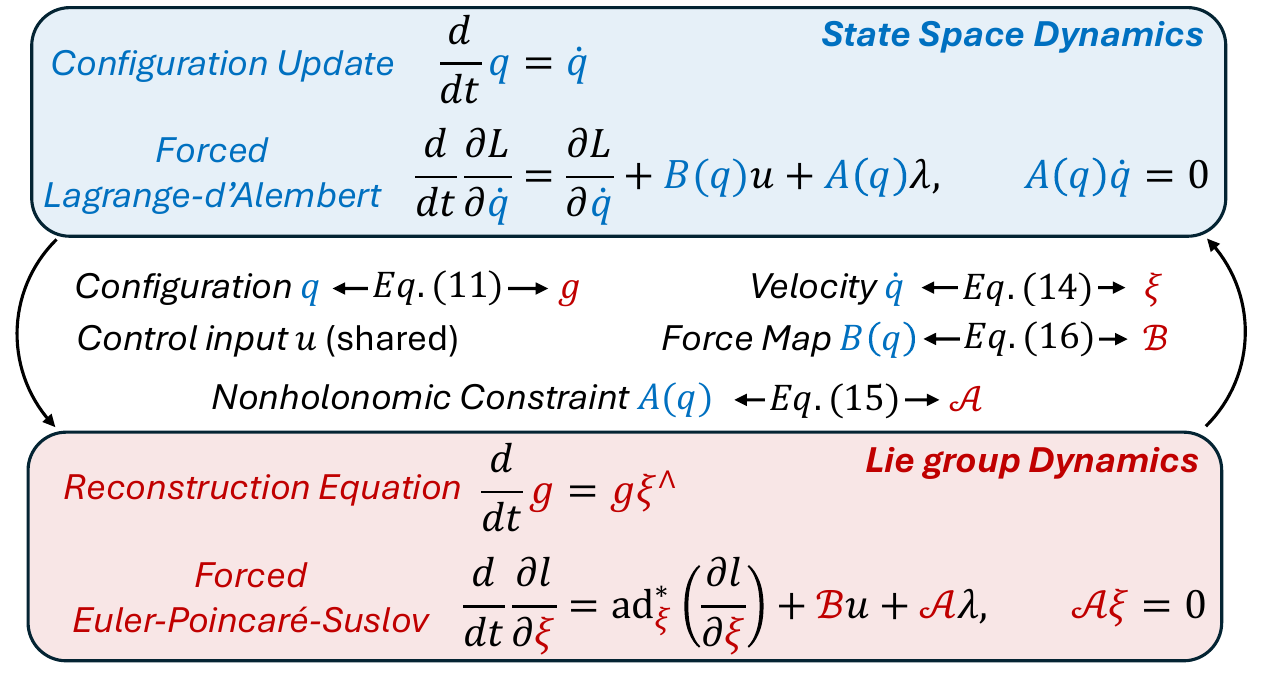}
  \caption{Converting nonholonomic systems between State Space and Lie group form. The Configuration Update and Reconstruction Eq. show how velocities impact configurations. The Lagrange-d'Alembert and Euler-Poincar\`e-Suslov Eqs. show how accelerations impact velocities. The Reconstruction Eq. and the $\textrm{ad}^*_\xi$ term introduce cross-dimensional couplings to the Lie group form, allowing rotational uncertainty to impact positional uncertainty.}
   \label{fig:APPapp-eq_relation}
\end{figure}

Let us demonstrate this general conversion process by applying it to the dynamical system used in the main paper: the \textit{second-order unicycle}.
\looseness-1 The configuration of this nonlinear underactuated nonholonomic system in generalized coordinates can be given by $q=[x^q,y^q, \theta^q]$, with C-space $\mathcal Q = \mathbb R^2 \times [-\pi, \pi)$.
The constraint matrix is $A(q)=[s_{\theta^q}, -c_{\theta^q}, 0]$, where $s$ is $\sin$ and $c$ is $\cos$.
The control inputs are body-fixed wrenches $u =[f_x^b, \tau^b_z]\transpose \in \mathbb R^2$, with force map $\inputMatrix(q)
=\begin{psmallmatrix}
    c_{\theta^q} & s_{\theta^q} & 0 \\ 0 &  0 & 1
\end{psmallmatrix}\transpose$$.$
In \textit{body-fixed frame}, with origin at the CoM, the configuration becomes $g\in SE(2)$ and the twists $\xi=[\dot x^b, \dot y^b,\dot \theta^b]$. The \textit{ad operator} is $\textrm{ad}_\xi = \begin{psmallmatrix}
    0 & -\dot \theta^b & \dot y^b \\ \dot \theta^b & 0 & -\dot x^b \\ 0 & 0 &0
\end{psmallmatrix}$.
We can then obtain
$\lTwistMap(q) 
= \begin{psmallmatrix}    c_{\theta^q} & s_{\theta^q} & 0 \\ -s_{\theta^q} & c_{\theta^q} & 0 \\ 0 & 0 & 1
\end{psmallmatrix}=
\begin{psmallmatrix}  R^{-1}(q) & 0 \\  0 & 1
\end{psmallmatrix}$.
The constraint matrix is $\lConstMatrix=[0, 1, 0]\transpose$, enforcing no side-slip velocity, and the force map $\lInputMatrix 
=\begin{psmallmatrix}
    1 & 0 &0 \\ 0 &0 &1
\end{psmallmatrix}\transpose$$.$

\subsection{Nonholonomic Constraint Pre-Computation for Euler-Poincar\'e-Suslov Dynamics Propagation} \label{sec:APP-trickSpeedUpComp}

\looseness-1 In both the Lagrange-d'Alembert and the EPS equations \eqref{eq:eps_motion}, nonholonomic constraints are often enforced via Lagrange multipliers $\lambda$. This generally implies solving an extra equation at every integration step to determine $\lambda$, or integrating an augmented system. However, for systems where the nonholonomic constraints can be defined in terms of the body-frame twists (e.g., the second-order unicycle), $\lConstMatrix$ is configuration-invariant (constant). This allows us to instead integrate the unconstrained (holonomic) system and project the resulting twists onto the nonholonomic constraint manifold using a pre-computed matrix $\mathcal P$, potentially speeding up computation. \newline \indent Starting from the Euler-Poincar\'e-Suslov equations
\begin{equation*}
\lInertiaMatrix \dot \xi = \textrm{ad}^*_\xi(\lInertiaMatrix \xi) + \lInputMatrix u + \lConstMatrix^\top \lambda,\quad \lConstMatrix\xi =0,
\end{equation*}
note that a constant $\lConstMatrix$ gives $\frac{d}{dt}(\lConstMatrix \xi)= 0 \Rightarrow \lConstMatrix \dot \xi = 0$. Then, re-arranging for $\lambda$, we obtain
\begin{align*}
 &\lInertiaMatrix \dot \xi = \textrm{ad}^*_\xi(\lInertiaMatrix\xi) + \lInputMatrix u + \lConstMatrix^\top \lambda\\
& \Leftrightarrow  \dot \xi = \lInertiaMatrix^{-1}(\textrm{ad}^*_\xi(\lInertiaMatrix\xi) +\lInputMatrix  u + \lConstMatrix^\top \lambda)\\
&\Leftrightarrow \lConstMatrix\dot \xi = \lConstMatrix\lInertiaMatrix^{-1} (\textrm{ad}^*_\xi(\lInertiaMatrix\xi) + \lInputMatrix u + \lConstMatrix^\top \lambda) = 0\\
&\Leftrightarrow \lConstMatrix\lInertiaMatrix^{-1} (\textrm{ad}^*_\xi(\lInertiaMatrix\xi) + \lInputMatrix u) + \lConstMatrix\lInertiaMatrix^{-1}\lConstMatrix^\top \lambda = 0\\
&\Leftrightarrow \lConstMatrix\lInertiaMatrix^{-1} (\textrm{ad}^*_\xi(\lInertiaMatrix\xi) + \lInputMatrix u) = - \lConstMatrix\lInertiaMatrix^{-1}\lConstMatrix^\top \lambda\\
&\Leftrightarrow - (\lConstMatrix\lInertiaMatrix^{-1}\lConstMatrix^\top)^{-1}\lConstMatrix\lInertiaMatrix^{-1} (\textrm{ad}^*_\xi(\lInertiaMatrix\xi) + \lInputMatrix u) =  \lambda.
\end{align*}

\looseness-1 For most systems $\lInertiaMatrix$ is positive definite. Then,
if the $k$ nonholonomic constraints are linearly independent, $\lConstMatrix$ has full row rank and $(\lConstMatrix\lInertiaMatrix^{-1}\lConstMatrix^\top)$ admits an inverse, ensuring $\lambda$ exists.
Plugging the expression for $\lambda$ into the EPS equations gives
\begin{align*}
&\lInertiaMatrix \dot \xi = \textrm{ad}^*_\xi(\lInertiaMatrix\xi) + \lInputMatrix u + \lConstMatrix^\top \lambda\\
&= \textrm{ad}^*_\xi(\lInertiaMatrix\xi) + \lInputMatrix u - \lConstMatrix^\top (\lConstMatrix\lInertiaMatrix^{-1}\lConstMatrix^\top)^{-1}\lConstMatrix\lInertiaMatrix^{-1} (\textrm{ad}^*_\xi(\lInertiaMatrix\xi) + \lInputMatrix u)\\
& = (I- \lConstMatrix^\top (\lConstMatrix\lInertiaMatrix^{-1}\lConstMatrix^\top)^{-1}\lConstMatrix\lInertiaMatrix^{-1} )(\textrm{ad}^*_\xi(\lInertiaMatrix\xi) + \lInputMatrix u)\\
&\Leftrightarrow \dot \xi = \lInertiaMatrix^{-1}(I- \lConstMatrix^\top (\lConstMatrix\lInertiaMatrix^{-1}\lConstMatrix^\top)^{-1}\lConstMatrix\lInertiaMatrix^{-1} )(\textrm{ad}^*_\xi(\lInertiaMatrix\xi) + \lInputMatrix u)
\end{align*}
Defining $\mathcal P := \lInertiaMatrix^{-1}\lConstMatrix^\top(\lConstMatrix\lInertiaMatrix^{-1}\lConstMatrix^\top)^{-1}\lConstMatrix$, we have
\begin{align*}
\lInertiaMatrix^{-1}(I- \lConstMatrix^\top (\lConstMatrix\lInertiaMatrix^{-1}\lConstMatrix^\top)^{-1}\lConstMatrix\lInertiaMatrix^{-1} )= \lInertiaMatrix^{-1}-\mathcal P\lInertiaMatrix^{-1}.
\end{align*}
Let $\dot \xi_{free} := \lInertiaMatrix^{-1}(\textrm{ad}^*_\xi(\lInertiaMatrix\xi) + \lInputMatrix u)$ denote the twists resulting from the Euler-Poincar\'e equations of motion for the holonomic system with same inertia. We can finally write the twist update equation for the nonholonomic EPS system as
\begin{align}
\dot \xi &= (1-\mathcal P)\lInertiaMatrix^{-1}(\textrm{ad}^*_\xi(\lInertiaMatrix\xi) + \lInputMatrix u)=(I-\mathcal P)\dot \xi_{free}.
\end{align}
Thus, for systems with constant $\lInertiaMatrix$ and $\lConstMatrix$, we can pre-compute and store $(1-\mathcal P)$, potentially speeding up numerical integration.
For the SS unicycle dynamics, $\constMatrix$ is configuration-dependent and so a similar process does not hold.

\subsection{Numerical Validation: Converting State Space Dynamics to Lie group Form}\label{sec:App-numericalValidationDyn}

\begin{wrapfigure}[11]{r}{0.3\linewidth}
    \centering
    \vspace{-6mm}
    \includegraphics[width=\linewidth]{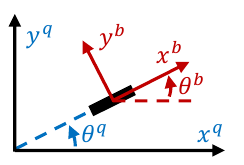}
    \vspace{-8mm}
    \caption{\textit{Generalized} $(\cdot)^q$ and \textit{body-fixed} $(\cdot)^b$ \textit{coordinates} of a second-order unicycle.}
    \label{fig:unicycle_coords}
\end{wrapfigure}
To validate the dynamics conversion described in Appendix~\ref{sec:APPconvert_to_EPS}, we performed numerical experiments on a deterministic ($u_{noise}=0$) second-order unicycle system with no model mismatch, whose configuration can be seen in Figure \ref{fig:unicycle_coords}. The constraint matrices $(\constMatrix, \lConstMatrix)$ and force maps $(\inputMatrix,\lInputMatrix)$ shown above allow us to write the continuous-time equations of motion in SS and Lie group form.
Given $\tilde \lInertiaMatrix$, one can convert commanded body-frame wrenches $u_{cmd}$ to commanded accelerations, which we report instead for interpretability.
We propagate both representations using Forward Euler (FE), Symplectic Euler (SE), Heun, and Runge-Kutta 2nd order (RK2) integrators.
We further implemented a RK4 integrator for SS, and since RK4 is not trivially transferable to Euler-Poincar\'e-Suslov dynamics \citeapp{munthe1998runge}, we used a Commutator-Free Magnus fourth-order integrator (CF4) \citeapp{blanes2006fourth} for comparison.
Since $\lConstMatrix$ is constant, we also use the pre-computation procedure described in Appendix~\ref{sec:APP-trickSpeedUpComp}. 
We simulated $625$ one-second-long trajectories per integrator, for both SS and Lie, spanning the grid defined by: $\dot x^b_0\in\texttt{lin}(0.1, 0.5, 5); \dot y^b_0=0;\dot \theta^b_0\in\texttt{lin}(0, 0.5, 5);\ddot x^b_0\in\texttt{lin}(0, 0.5, 5);\ddot \theta^b_0\in\texttt{lin}(0, 2, 5)$.
Testing only positively-valued angular rates and actions is sufficient due to the systems' inherent symmetry.
We compared both the SS and Lie integrated trajectories with a high-fidelity integration-free reference method \citeapp{missura2011efficient}. For the constant acceleration trajectories we tested, \citeapp{missura2011efficient} provides a closed-form solution to the second-order unicycle motion if $\ddot \theta_b=0$, and a Fresnel integral-based approximation otherwise.
We use as accuracy metric the RMSE of $e_1$, taken between the configurations of the reference $g_1$ and the numerically integrated SS/Lie dynamics $\tilde g_1$ at the end of the trajectory. 
In Table III, we show the average accuracy and computation time for $\dt = 0.1$ sec.
\begin{table}[h!]
\captionsetup{              
  skip   = 4pt, 
}
\centering
\scriptsize
\setlength{\tabcolsep}{3pt}
{\setstretch{0.8}\caption{Average Time per integration step (over 50k calls) and Accuracy (over $625$ grid-spanning trials) for $\dt=0.1\textrm{s}$}}
\label{tab:app-integrator_times_means}
\begin{tabular}{l|c|c|c|c|c|c|c}
\toprule
\multicolumn{1}{c|}{\multirow{2}{*}{Performance Metric}} &
\multicolumn{1}{c|}{\multirow{2}{*}{Dynamics Representation Space}} & \multicolumn{6}{c}{Numerical Integrator} \\
& & FE & SE & Heun & RK2 & CF4 & RK4 \\
\midrule
\multirow{2}{*}{Runtime Per Step (ms) $\downarrow$} & State Space & 0.203 & 0.203 & 0.415 & 0.407 & --  & 0.847\\ 
& Lie group & \textbf{0.185} & \textbf{0.184} & \textbf{0.379} & \textbf{0.369} & \textbf{0.306} & -- \\
\midrule
\multirow{2}{*}{Accuracy (RMSE of $e_1 $) $\downarrow$} & State Space & 3.7e-2 & \textbf{3.4e-2} & 6.9e-4 & 8.7e-4 & -- & 4.0e-7 \\ 
& Lie group & \textbf{3.4e-2} & \textbf{3.4e-2} & \textbf{1.4e-4} & \textbf{1.4e-4} & \textbf{1.2e-8} & --\\
\bottomrule
\end{tabular}
\end{table}

\looseness-1 The computation time of each integrator was, as expected, closely proportional to the number of times $\approxDynamics$ is called. CF4 was 2.7 times faster than RK4 possibly due to only requiring a single call to $\approxDynamics$.
In order to estimate each integrator's order of accuracy, we repeated the experiment above for seven different log-spaced $\dt$ between $0.001$ and $0.1$ seconds.
We found the empirical orders of accuracy to align with theory for both the SS and Lie form --- FE and SE were first-order, Heun and RK2 second-order, RK4 and CF4 fourth-order. 
For each integrator and $\dt$ tested, the accuracy of the Lie group form dynamics was comparable to that of the SS dynamics. 
These results support the presented conversion from SS to Lie group dynamics.

\subsection{System Identification: Estimating mass and inertia for the MBot and JetBot}\label{sec:APP-systemIdentification}
\begin{wrapfigure}[13]{r}{0.4\linewidth}
    \centering
    \vspace{-3mm}
    \includegraphics[width=\linewidth]{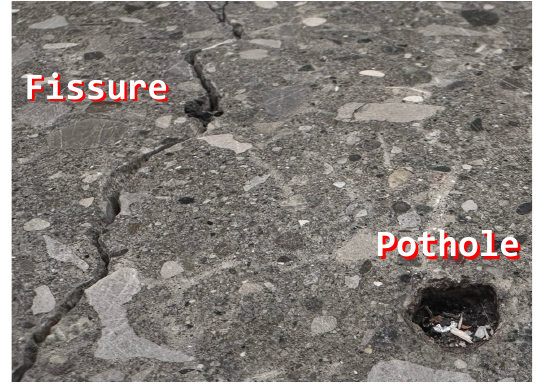}
    \caption{The lab floor is uneven, potentially introducing aleatoric uncertainty.}
    \label{fig:app-aleatoric-sources}
\end{wrapfigure}

\textit{Hardware}: The MBot's mass was estimated using a weight scale and its moment of inertia using the solid-disk's formula $\tilde I_z = \frac{1}{2}\tilde m \tilde r^2$, where $\tilde r$ is an estimated body radius. While crude, epistemic uncertainty arising from this system identification is also quantified by our proposed algorithm. Aleatoric uncertainty could have originated from network jitter, ground surface imperfections (Fig. \ref{fig:app-aleatoric-sources}), CoM variations (battery was unsecured), etc.

\textit{Simulation}: The JetBot's mass was estimated by applying constant body-frame forces/torques from rest, and then linearly fitting the observed linear/angular accelerations. This yielded $\tilde \lInertiaMatrix=\textrm{diag}(2.8, 2.8, 0.007)$.
Control inputs $u_{cmd}$ were converted to joint-efforts $\tau_{wheels}$ via the kinematic relation $\tau_{wheels} =J_{jetbot}u_{cmd}$, with $J_{jetbot}=\left[\begin{smallmatrix}
\tilde r/2 & -\tilde r/\tilde b \\ \tilde r / 2 & \tilde r/ \tilde b
\end{smallmatrix}\right]$. We used the wheel radius $\tilde r$ and the wheel separation distance $\tilde b$ provided by Isaac Sim's documentation. Epistemic uncertainty introduced by this transformation is accounted for by \textbf{CLAPS}. Aleatoric uncertainty was artificially introduced, as detailed in $\S$\ref{sec:sim_exp_only}.

\subsection{Sensitivity of Mesh Reconstruction to the Number of Sampled Boundary Points}\label{sec:APP-sensitivityNrParticles}

While \textbf{CLAPS} can be used for safe planning without reconstructing a C-space mesh (see $\S$\ref{sec:imp_tricks}), we have assessed the computation speed vs. $\mathcal C^q$ reconstruction accuracy tradeoff for our Python-based implementation${}^\text{\ref{code_url}}$ of Alg. \ref{alg:fit-mesh}. Using the calibration and validation data from $\S$\ref{sec:real_experiment}, we evaluated \textbf{CLAPS} + Alg. \ref{alg:fit-mesh} when $5000$, $2000$, $1000$, and $500$ particles are sampled to represent the C-space surface. For each case, we report the average runtime, C-space volume, and marginal coverage in Table IV.

\begin{table}[h]
\captionsetup{skip=4pt}
\centering
\setlength{\tabcolsep}{2pt}
{\setstretch{0.8}\caption{$\#$ Sampled $\partial \mathcal C^q$ Points' Impact on Computed Metrics (\textbf{CLAPS})}}
\label{tab:app_pts_vs_speed}
\scriptsize
\begin{tabular}{c|c|c|c|c}
\toprule
\makecell{Number of $\partial \mathcal C^q$ \\ Points Sampled} & \makecell{5000 \\ (results in $\S$\ref{sec:real_experiment})} & 2000 & 1000 & 500 \\
\midrule
\makecell{CLAPS (Online Part) + \\ Alg. \ref{alg:fit-mesh} (s) $\downarrow$} 
& \ccell{green!15}0.37 $\pm$ 0.02 & \ccell{green!15}0.11 $\pm$ 0.01 & \ccell{green!15}0.06 $\pm$ 0.01 & \textbf{\ccell{green!15}0.04 $\pm$ 0.00} \\
\midrule
\makecell{C-space Region Volume\\ ($m^2 \cdot rad$) $\downarrow$} 
& \ccell{green!15}0.00211 & \ccell{green!15}0.00210 & \ccell{green!15}0.00208 & \textbf{\ccell{green!15}0.00205} \\
\midrule
\makecell{Marginal Coverage \\(\%) $\uparrow$} 
& \ccell{green!15}90.41 & \ccell{green!15}90.36 & \ccell{green!15}90.31 & \ccell{green!15}90.22 \\
\bottomrule
\end{tabular}
\end{table}
\looseness-1 It appears that reducing the number of samples used to approximate the C-space region's surface can provide an order-of-magnitude improvement to runtime, without significantly impacting the coverage or volume.
At $500$ sampled points, our online implementation runs at $25$ Hz, the same rate at which the MBot receives sensor measurements. Hence, the provided implementation appears adequate for deployment.
The small observed change in C-space volume, and hence marginal coverage, will depend on the mesh reconstruction algorithm used, which is outside the scope of this letter.

\bibliographystyleapp{IEEEtran}
\bibliographyapp{references}

\end{document}